%% file: arxiv.tex
\documentclass[11pt]{article}

\date{}
\usepackage{times}

\usepackage[utf8]{inputenc} % allow utf-8 input
\usepackage[T1]{fontenc}    % use 8-bit T1 fonts
\usepackage{hyperref}       % hyperlinks
\usepackage{url}            % simple URL typesetting
\usepackage{booktabs}       % professional-quality tables

\usepackage{amsfonts}       % blackboard math symbols
\usepackage{amsmath}
\usepackage{amsthm}
\usepackage{amssymb}
\usepackage{mathtools}

\usepackage{nicefrac}       % compact symbols for 1/2, etc.
\usepackage{microtype}      % microtypography
\usepackage{xspace}
\usepackage{xcolor}
\usepackage{xfrac}

\usepackage{algorithm}
\usepackage[noend]{algpseudocode}
\usepackage{multicol,multirow}

\usepackage[numbers, compress]{natbib}
\usepackage{caption}
\usepackage{enumitem}

\usepackage{graphicx}
\usepackage{float}

\newtheorem{theorem}{Theorem}

\newtheorem{lemma}[theorem]{Lemma}
\newtheorem{definition}{Definition}
\newtheorem{proposition}{Proposition}

\input{define.tex}

\title{Vertical Symbolic Regression}
\author{Nan Jiang, Md Nasim, Yexiang Xue\\
Department of Computer Science, Purdue University \\
\texttt{\{jiang631,mnasim,yexiang\}@purdue.edu}}

\begin{document}

\maketitle

\begin{abstract}
\input{tex/0.abstract}

\end{abstract}

\section{Introduction}

\input{tex/1.intro}
\input{tex/2.prelim}

\input{tex/3.method}

\input{tex/4.regressor}

\input{tex/6.theory}

\input{tex/7.related}

\section{Experiments}

\input{tex/8.1.exp-set}

\input{tex/8.2.exp}

\input{tex/8.3.ablation}

\section{Conclusion}
\input{tex/9.conclusion}

\section*{Acknowledgments}
This research was supported by NSF grant CCF-1918327 and DOE – Fusion Energy Science grant: DE-SC0024583.

\clearpage
\bibliography{reference}
\bibliographystyle{unsrtnat} 
\newpage

\appendix
\input{tex/10-expset}

\input{tex/10.dataset}

\end{document}

%% file: define.tex
\usepackage{xspace}
\usepackage{bm}

\algrenewcommand\algorithmicrequire{\textbf{Input:}}
\algrenewcommand\algorithmicensure{\textbf{Output:}}

\newcommand{\method}{\text{VSR}\xspace}
\newcommand{\cvgp}{\text{VSR-GP}\xspace}
\newcommand{\gp}{\text{GP}\xspace}
\newcommand{\cvmt}{\text{VSR-MCTS}\xspace}
\newcommand{\mcts}{\text{MCTS}\xspace}

%% file: tex/0.abstract.tex
Automating scientific discovery has been a grand goal of Artificial Intelligence (AI) and will bring tremendous societal impact.
Learning symbolic expressions from experimental data is a vital step in AI-driven scientific discovery.
Despite exciting progress, most endeavors have focused on the \textit{horizontal} discovery paths, \textit{i.e.}, they directly search for the best expression in the full hypothesis space involving all the independent variables.
Horizontal paths are challenging due to the exponentially large hypothesis space involving all the independent variables. 
We propose \underline{\textbf{V}}ertical \underline{\textbf{S}}ymbolic  \underline{\textbf{R}}egression (\method) to expedite symbolic regression.
 The \method starts by fitting simple expressions involving a few independent variables under controlled experiments where the remaining variables are held constant. 
 It then extends the expressions learned in previous rounds by adding new independent variables and using new control variable experiments allowing these variables to vary. 
 The first few steps in vertical discovery are significantly cheaper than the horizontal path, as their search is in reduced hypothesis spaces involving a small set of variables. As a consequence, vertical discovery has the
potential to supercharge state-of-the-art symbolic regression approaches in handling complex equations with many contributing factors.
Theoretically, we show that the search space of \method can be exponentially smaller than that of horizontal approaches when learning a class of expressions.
Experimentally, 
\method outperforms several baselines in learning symbolic expressions involving many independent variables.

%% file: tex/1.intro.tex
Automating scientific discovery has been a grand goal of Artificial Intelligence (AI) dating back to its founders~\citep{langey1988scientificdiscovery,kulkarni1988processes,wang2023scientific};
but, it remains a holy grail. 
The underlying societal impact is immense because of its multiplier effect.
This work attacks a fundamental problem in AI-driven scientific discovery -- symbolic regression, namely, 
learning physics laws in the form of symbolic expressions from experimental data.
Much progress has been made in this domain, such as search-based methods~\citep{LANGLEY1981DataDiscovery,LENAT1977ubiquity},  genetic programming~\citep{doi:10.1126/science.1165893,DBLP:conf/gecco/VirgolinAB19,DBLP:conf/gecco/HeLYLW22}, reinforcement learning~\citep{DBLP:conf/iclr/PetersenLMSKK21,DBLP:conf/nips/MundhenkLGSFP21,DBLP:conf/iclr/PetersenLMSKK21}, deep function approximation~\citep{mcconaghy2011ffx,DBLP:conf/icnc/ChenLJ17,Raissi20Fluid,RAISSI2019PhysicsInformedNN,Liu21AIPoincare,nanovoid_tracking,chen2018neural,jumper2021highly,brunton2016sparse}, and
integrated systems~\citep{Valdes1994,king2004functional,king2009autosci,Lintott2008Galaxy}.

Most endeavors focus on \textit{horizontal} discovery paths, i.e., they directly search for the best equation in the full hypothesis space involving all independent variables (red path in Fig.\ref{fig:hv}). The horizontal search can be challenging because of the exponentially large space of possible expressions.
As a result, state-of-the-art approaches are limited to learning simple expressions composed of a small number of independent variables.
Expressions involving many independent variables are still beyond reach. 
After the conventional wisdom of training deeper models and with more data stretched to its extremity in the horizontal search, what is the next paradigm-changing idea?

\begin{figure}[!t]
    \centering
    \includegraphics[width=0.99\textwidth]{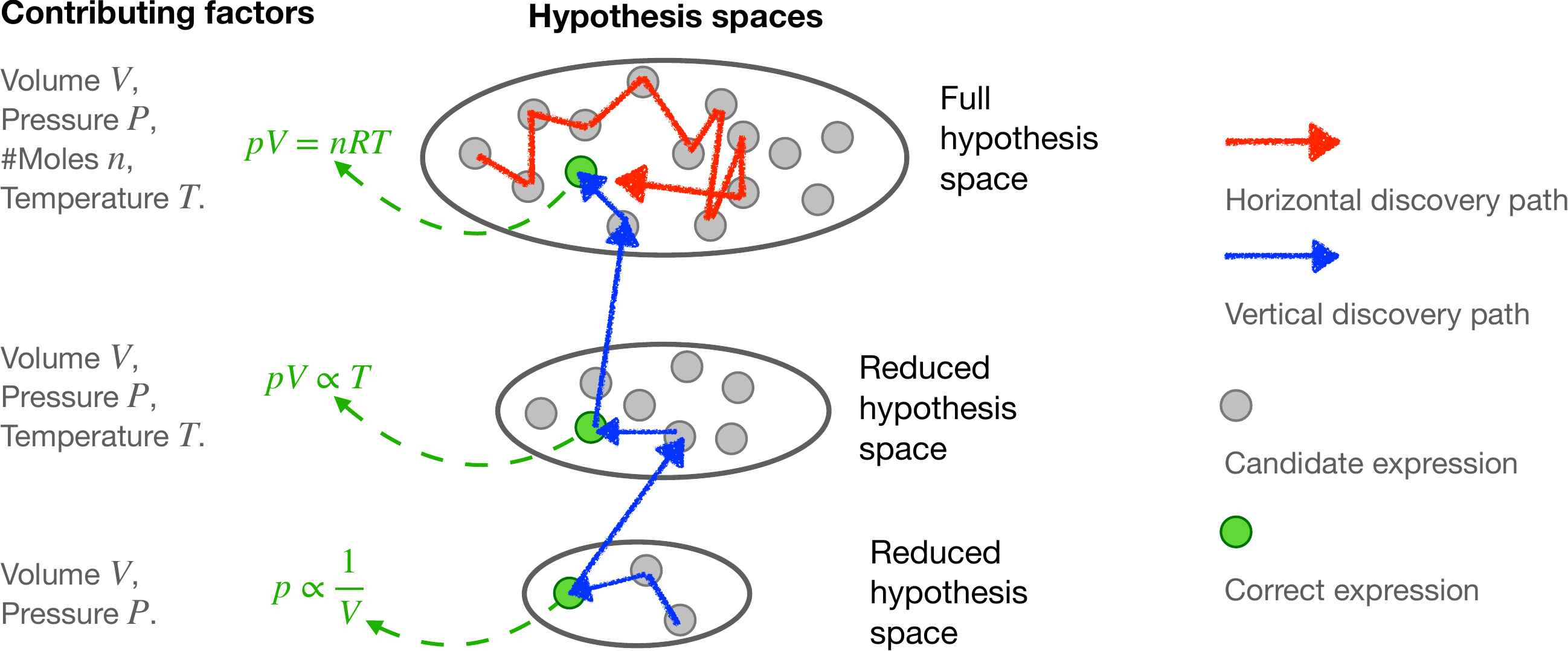}
    \caption{Vertical discovery paths (colored blue) scale up scientific discovery better than the horizontal discovery path (colored red). 
    Here, we demonstrate two paths in the discovery of the ideal gas law.  Vertical discovery starts by finding the relationship between two contributing factors ($p, V$) in a reduced hypothesis space while holding other factors constant.
    It then finds models in (extended) reduced hypothesis space with three factors ($p, V, T$), and finally in the full hypothesis space. 
    Searching following the vertical paths is way cheaper because the sizes of the reduced hypothesis spaces are exponentially smaller than the full hypothesis space. 
    In comparison, horizontal discovery, which directly searches in the full hypothesis space, can be costly.} \label{fig:hv}
\end{figure}

Interestingly, the \textit{vertical} discovery paths have been largely overlooked in AI.  
A successful example is the discovery of the ideal gas law $pV = nRT$ (shown in Fig.~\ref{fig:hv}).
% For example, the historical discovery of the ideal gas law 
%
In this example, scientists first held $n$ (gas amount) and $T$ (temperature) as constants and found that $p$ (pressure) is inversely proportional to $V$ (volume)~\citep{boyle19651662}. 
They then found that $pV$ is proportional to $T$ when the amount of gas $n$ is held constant~\citep{gay1802expansion}.
Finally, they allowed $n$ to vary and found the final equation $pV = nRT$~\citep{avagadro1811essai}. 
 This led to a vertical discovery path (blue path in Fig.~\ref{fig:hv}).
The first few steps of a vertical path can be significantly cheaper than the horizontal path, because
the searches are in reduced hypothesis spaces involving a small number of independent variables. 
As a result, the vertical discovery path has the potential to supercharge state-of-the-art approaches for uncovering complex scientific phenomena with more contributing factors than current approaches can handle.

\begin{figure}[!t]
    \centering
    \includegraphics[width=1.\textwidth]{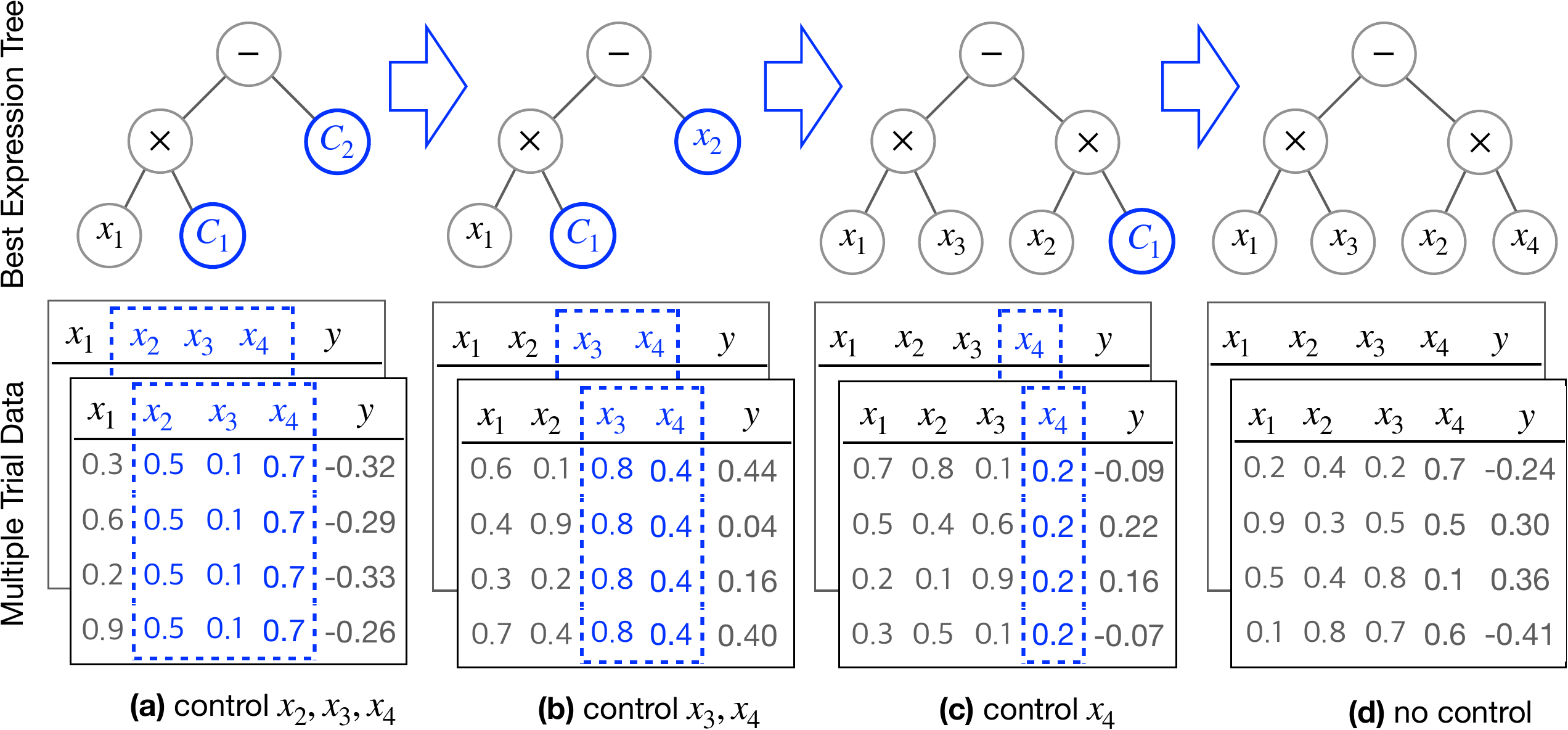}
    \caption{Example showing the process for \method to discover an expression with $4$ input variables. 
    \textbf{(a)} Initially, a reduced-form equation $\phi'=C_1 x_1 - C_2$ is found via fitting control variable data in which  $x_2, x_3, x_4$ are held as constants and only  $x_1$ is allowed to vary.  Two leaf nodes $C_1$ and $C_2$ (colored blue) are summary constants, which are sub-expressions containing the controlled variables.
    \textbf{(b)} This equation is expanded to $C_3 x_1 - C_4 x_2$ in the second stage by fitting the data in which only $x_3 and x_4$ are held constant.
    \textbf{(c,d)} This process continues until the ground-truth equation $\phi=x_1 x_3 - x_2 x_4$ is found. At the bottom, the data for the control variable experiment in each stage requires the controlled variables to take the same values. }\label{fig:running-example}
\end{figure}

We propose \underline{\textbf{V}}ertical \underline{\textbf{S}}ymbolic \underline{\textbf{R}}regression (\method), which implements the idea of vertical scientific discovery for symbolic regression tasks. 
The key insight of \method is to expand from reduced-form equations, which model a subset of independent variables, to full equations, adding one new independent variable at a time.  
The reduced-form equations are learned from \textit{a customized set of control variable experiments}, in which non-participating independent variables are held as constants. 
In other words, \method requires learning from a customized set of control variable experiments. 
This is in contrast to the current learning paradigm of most symbolic regression approaches, in which individuals learn from a fixed dataset collected a priori. 

The general procedure of \method is as follows. 
First, \method holds all independent variables except for one as constants and learns an expression that maps the single variable to the dependent variable using a symbolic regressor (shown in Fig.~\ref{fig:running-example}(a)). 
Mapping the dependence of one independent variable is easy. Hence a symbolic regressor can usually 
recover the ground-truth reduced-form equation. 
Then, \method frees one more independent variable. In each iteration, the regressor is used to modify the equations learned in the previous rounds to incorporate the new independent variable (shown in Fig.~\ref{fig:running-example}(b,c)).
This procedure repeats until all the independent variables have been incorporated into the symbolic expression (shown in Fig.~\ref{fig:running-example}(d)). 
The regressors used to modify the equations in each round can be regular symbolic regressors. In this paper, we use Genetic Programming (GP) and Monte Carlo Tree Search (MCTS).

\method heavily depends on control variable experimentation, 
which is a classic procedure widely used in science; however \method is largely overlooked in AI~\citep{lehman2004designing,DBLP:books/daglib/0022270}. 
In science, control variable experiments are used in the analysis of complex scientific phenomena involving many contributing factors. In control variable experiments, a subset of the contributing factors are held constant (i.e., controlled variables). The dependence is studied in the \textit{reduced} hypothesis space involving the remaining factors. 
The result is a \textit{reduced-form} expression that models the relationship only among the non-controlled variables (i.e., free variables). 

Theoretically, we show that \method as an incremental builder can reduce the exponential-sized hypothesis space for candidate expressions into a polynomial-sized space when fitting a class of symbolic expressions. 
In the experiments, we implement two variants of the proposed \method algorithms, \textit{namely},  \method using Genetic Programming (\cvgp) and \method using Monte Carlo Tree Search (\cvmt). 
We show that \cvgp and \cvmt not only outperform the original GP and MCTS algorithms but also outperform several state-of-the-art approaches on the symbolic regression task. 
More specifically, \method yields the expressions with the smallest median Normalized Mean-Square Errors (NMSE) 
among all 7 competing approaches on noiseless datasets (in Table~\ref{tab:Trigonometric-nmse-noiseless} and Table~\ref{tab:feynman-livermore2-nmse-noiseless}) and 
20 noisy benchmark datasets (in Table~\ref{tab:Trigonometric-nmse-noisy}). In general, \cvgp attains the best empirical results on datasets with a large number of variables while \cvmt is the best on datasets with a median number of variables.
Evaluating simpler equations, we show that our \method takes less training time and memory, but has a higher rate of recovering the ground-truth expressions compared to baselines following the horizontal paths (in Table~\ref{tab:recovery}). 
We also demonstrate that our \method is consistently better than the baselines under different evaluation metrics (in Fig.~\ref{fig:evalucate-metric}),  different quantiles (25\%, 50\% and 75\%) of the NMSE metric (in Fig.~\ref{fig:Quartile-trig-nmse-noiseless-partial}), and with different amounts of Gaussian noise added to the data (in Fig.~\ref{fig:noise-level-metric}). 

Our contributions can be summarized into the following points:
\begin{enumerate}
    \item We propose Vertical Symbolic Regression (\method) for symbolic regression tasks with many independent variables. 
    \method starts from learning equations mapping a subset of independent variables in the reduced hypothesis spaces through a customized set of control variable experiments. Then \method iteratively expands such equations to the full hypothesis space. 
    \item Because the first few steps following the vertical discovery route can be exponentially less expensive than discovering the equation in the full hypothesis space (horizontal route), \method has the potential to supercharge state-of-the-art approaches in uncovering complex scientific equations with more contributing factors than what current approaches can handle.
    \item Theoretically, we show that \method can reduce exponential-sized hypothesis spaces for symbolic regression to polynomial-sized spaces when searching for a class of symbolic expressions. 
    \item In the experiments,  we implement two variants of the proposed \method, VSR using Genetic Programming (\cvgp) and VSR using  Monte Carlo Tree Seach (\cvmt)\footnote{The collected symbolic regression datasets, our code implementation of \cvgp and \cvmt and the baselines considered for comparison, are publicly accessible at: \url{https://bitbucket.org/jiang631/scibench}.}. Empirically, we demonstrate that \cvgp and \cvmt outperform state-of-the-art symbolic regression approaches in discovering multi-variable equations from data. 
    VSR finds the expressions with the smallest median NMSEs among all 7 competing approaches on noiseless datasets and 20 noisy benchmark datasets.
\end{enumerate}

\medskip
\noindent\textbf{Connection to Existing Methods.}
Our \method is closely connected to a line of research that also implemented the human scientific discovery process using AI, pioneered by the BACON systems~\citep{DBLP:conf/ijcai/Langley77,DBLP:conf/ijcai/Langley79,DBLP:conf/ijcai/LangleyBS81,king2004functional,king2009autosci,cerrato2023rlsci}. 
Our \method uses modern machine-learning approaches such as genetic programming while BACON's discovery was driven by rule-based engines. Indeed, both approaches share a common vision – the {integration of experiment design and model learning} can further expedite scientific discovery. 

%% file: tex/2.prelim.tex
\section{Preliminaries} \label{sec:prelim}
\textbf{Symbolic Expression.} 
Let $\mathbf{x}$ be a set of input variables and $\mathbf{c}$ be a set of constants. 
A symbolic expression $\phi$ is expressed as variables $\mathbf{x}$ 
and constants $\mathbf{c}$ connected by mathematical operators. The mathematical operators can be $+,-,\times,\div$, etc. 
For example, $\phi= 5.8\times x_1\times x_2 \div x_3$ is a symbolic expression with three variables $\{x_1, x_2, x_3\}$, one constant $5.8$ and two operators $\{\times,\div\}$. The semantic meaning of a
symbolic expression follows its standard definition in arithmetic.
A symbolic expression can also be represented as an \textit{expression tree}, where variables and constants correspond to leaves, and operators correspond to the inner nodes of the tree. See, e.g., the expression tree in Figure \ref{fig:running-example}.
In addition to the expression tree, a symbolic expression can also be represented using context-free grammar~\citep{DBLP:conf/icml/TodorovskiD97}. The expression representations are tightly connected to the algorithms presented in this work. 
To avoid delay in the introduction of the main algorithm, we defer its introduction to Section~\ref{sec:regressor}.

\medskip
\textbf{Symbolic Regression.} 
Let ${D}=\{(\mathbf{x}_1, y_1), (\mathbf{x}_2, y_2)\ldots, (\mathbf{x}_n, y_n)\}$ be a dataset with $n$ samples, where $\mathbf{x}_i\in \mathbb{R}^m$ represents $m$ input variables and $y_i\in\mathbb{R}$ represents the output.
Given the dataset $D$,
the objective of symbolic regression (SR) is to find an optimal symbolic expression that best fits the dataset. A loss function $\ell$ is used to compare the fitness of different candidate expressions to dataset $D$. The optimal expression $\phi^*$ minimizes the average of this loss function:
\begin{equation} \label{eq:objective}
\phi^*\leftarrow\arg\min_{\phi\in \Pi}\;\frac{1}{n} \sum_{i=1}^n \ell(\phi(\mathbf{x}_i), y_i),
\end{equation}
% in addition to regularizers. 
Here, $\Pi$ is the hypothesis space--in other words, the set of all possible expressions.
As $\Pi$ can be exponentially large, finding the optimal expression in $\Pi$ is challenging and has been shown to be NP-hard~\citep{journal/tmlr/virgolin2022}.

Genetic programming (GP) and Monte Carlo Tree Search (MCTS) are two popular algorithms that can be used for symbolic regression.
The high-level idea of GP is to maintain a pool of candidate symbolic expressions, and iteratively improve this pool using selection, mutation, and crossover operations~\citep{DBLP:journals/gpem/UyHOML11}.
In each generation, candidate expressions undergo mutation and crossover operations probabilistically. Then in the selection step, those with the highest fitness scores, measured by how each expression predicts
the output from the input, are selected as the candidates for the next generation, together with a few randomly chosen ones to maintain diversity. After several generations, expressions with high fitness scores (\textit{i.e.}, those that fit the data well) survive in the pool of candidate solutions. The MCTS algorithm is a heuristic search algorithm applied to the space of all the expressions. 
The key concept of MCTS is to interpret mathematical operations and input variables by computational rules and symbols, establish symbolic reasoning of mathematical formulas via search trees, and employ a Monte Carlo tree search (MCTS) agent to find the optimal expression over all possible expressions, which is maintained by the search tree.
The MCTS agent obtains an optimistic selection policy through the traversal of search trees, featuring the path that maps to the arithmetic expression of the governing expressions~\citep{DBLP:conf/icml/KamiennyLLV23,DBLP:conf/iclr/Sun0W023}. 
Our \method is a general idea for symbolic regression following vertical discovery paths. In this process, both GP and MCTS can be used to identify reduced-form equations in each step. A detailed description of these two methods is presented in Section~\ref{sec:regressor}.

%% file: tex/3.method.tex
\section{Vertical Symbolic Regression}

We present our vertical symbolic regression solver in this section. 
Our approach first discovers a reduced-form expression that maps one of the input variables to the output with the values of the remaining variables being controlled. 
See Fig.~\ref{fig:running-example}(a) for an example of the first step. 
Then our method introduces one free variable at a time and expands the equation learned in the previous round to include this newly added variable. 
This process continues until all the variables are considered. 
See Fig.~\ref{fig:running-example}(b,c,d) for a visual demonstration. 
This approach is different from most state-of-the-art symbolic regression approaches, where those baselines learn the mapping between the output and all input variables from a fixed dataset collected a-priori.
Before we dive into the algorithm description, we first need to study what are
the outcomes of a control variable experiment and what conclusions we can draw
on the symbolic regression expression by observing such outcomes.

\subsection{Control Variable Experimentation}
\label{sec:ctrl_exp}

A control variable experiment, noted as $\texttt{CvExp}({\phi}, \mathbf{v}_c, \mathbf{v}_f, $ $\{T_k\}_{k=1}^K)$,
consists of the trial symbolic expression $\phi$, 
a set of controlled variables $\mathbf{v}_c$, 
a set of free variables $\mathbf{v}_f$, and $K$ trial experiments $T_1, \ldots, T_K$. 
The expression ${\phi}$ may have zero or multiple \textit{open constants}. 
The values of the open constants are determined by fitting the equation to the batch of data using a gradient-based optimizer, such as BFGS~\citep{bfgs}. 
To avoid abusing notation, we also use $T_k$ to denote the batch of data.

\begin{figure}[!t]
    \centering
    \includegraphics[width=\textwidth]{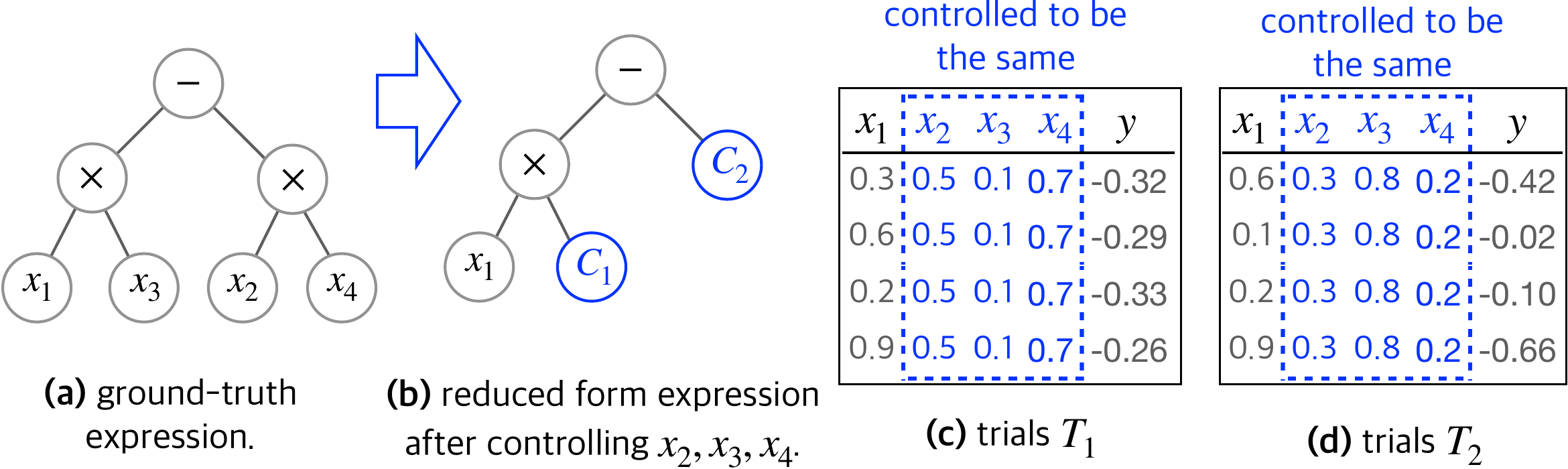}
    \caption{An example of two trials of a control variable experiment. 
    \textbf{(a)} The data of the experiment is generated by the ground-truth 
    expression $\phi=x_1x_3-x_2x_4$. \textbf{(b)} If we control
    $\mathbf{v}_c=\{x_2, x_3, x_4\}$ and only allow $\mathbf{v}_f=\{x_1\}$ to vary, it \textit{looks like}  
     the data are generated from the reduced-form equation $\phi'=C_1x_1-C_2$. \textbf{(c, d)} The generated data in two trials of the control variable experiments.
    The controlled variables are fixed within each trial but vary across trials.}
    \label{fig:reduced}
\end{figure}

\medskip
\textbf{Single Trial of a Control Variable Experiment.} 
A single trial of a control variable experiment fits the symbolic expression $\phi$ with a batch of data $T_k$. 
%
% To avoid abusing notations, we also use $T_k$ to denote the batch of data.
%
In the generated data $T_k$, every controlled variable is fixed to the same value while the free variables are set randomly. 
We assume that the values of the dependent variables in a batch are (noisy observations of) the ground-truth expressions with the values of the independent variables set in the batch. 
In scientific discovery, this step is achieved by conducting real-world experiments, \textit{i.e.}, controlling independent variables, and performing measurements on the dependent variable. 

For example, Fig.~\ref{fig:reduced}(c,d) demonstrates two trials ($K=2$) of a control variable experiment in which variables $x_2, x_3, x_4$ are controlled, \textit{i.e.}, $\mathbf{v}_c=\{x_2,x_3,x_4\}$. 
They are fixed to one value in trial $T_1$ (in Fig.~\ref{fig:reduced}(c)) and another value in trial $T_2$ (in Fig.~\ref{fig:reduced}(d)).
$x_1$ is the only free variable, \textit{i.e.}, $\mathbf{v}_f=\{x_1\}$. 

\medskip
\textbf{Reduced-form  Expression in a Control Variable Setting.} 
We assume that there is a ground-truth symbolic expression that produces the experimental data. 
In other words, the observed output is the execution of the ground-truth expression from the input, possibly in addition to some noise.
In control variable experiments, because the values of controlled variables are fixed in each trial, 
what we observe is the ground-truth expression in its \textit{reduced form}, where sub-expressions 
involving only controlled variables are replaced with constants.

Fig.~\ref{fig:reduced}(b) provides an example of the reduced form expression. 
Assume the data is
generated from the ground-truth expression in Fig.~\ref{fig:reduced}(a): $\phi=x_1 x_3 - x_2 x_4$.
When we control the values of the variables in $\mathbf{v}_c=\{x_2, x_3, x_4\}$, the data  \textit{looks like} they are generated from the \textit{reduced} expression: $\phi'=C_1 x_1 - C_2$. 
We can see that both $C_1$ and $C_2$ hold constant values in each trial. 
However, their values vary across trials because the values of the controlled variables change.
In trial $T_1$, when $x_2$, $x_3$, and $x_4$ are fixed to 0.5, 0.1, and 0.7,  $C_1$ takes the  value of $x_3$, \textit{i.e.}, 0.1. $C_2$ takes the value of $x_2 x_4$, \textit{i.e.}, 0.35. In trial $T_2$, $C_1=0.8$ and  $C_2=0.06$.

We call constants that represent sub-expressions involving controlled variables in the ground-truth expression the \textbf{\textit{summary constants}} and refer to constants in the ground-truth expression the \textbf{\textit{stand-alone constants}}. 
For example, $C_1$ in Fig.~\ref{fig:reduced}(b) is a summary constant, because it represents $x_3$ in the ground-truth expression.

\medskip
\textbf{Outcome of a Single Trial.} 
Given a batch of data $T_k$, the outcome of a single trial is a tuple $\langle o, \mathbf{c}, {\phi} \rangle$, where 1)  the fitness score $o\in\mathbb{R}$ measures the goodness-of-fit of candidate expression $\phi$.
One typical fitness function is the mean squared error (MSE). See Equation~\eqref{eq:loss-function} for the exact definitions of the MSE and other relevant fitness functions. 2) vector $\mathbf{c}\in\mathbb{R}^L$ is the values of all $L$ constants in the expression $\phi$ that best fit the data. 3) the predicted symbolic expression $\phi$.
For the example in Fig.~\ref{fig:reduced}, if we fit the reduced expression in (b) to the data in trial $T_1$, the best-fitted values are $\mathbf{c}_1=[0.1, 0.35]$ and the fitted expression is $\phi'_1=0.1 x_1- 0.35$. For trial $T_2$, the best-fitted values are  $\mathbf{c}_2=[0.8, 0.06]$ and the fitted expression is $\phi'_2=0.8x_1-0.06$. 
In both trials, the fitness scores (i.e., the MSE value) are $0$, indicating no errors.

\medskip
\textbf{Outcome of Multiple Trials.} 
We let the values of the control variables vary across different trials.
This corresponds to changing experimental conditions in the real-world scientific discovery process.
The outcomes of an experiment with $K$ trials are a tuple $\langle \mathbf{o}, C,\bm{\phi}'\rangle$. In this tuple, $\mathbf{o}=(o_1, \ldots, o_K)$ is the fitness score vector, where each $o_k\in\mathbb{R}$ is the fitness score of the $k$-th trial. ${C}=(\mathbf{c}_1, \ldots, \mathbf{c}_K)$ is the matrix of the fitted constants. Each row vector $C_{k,:}\in\mathbb{R}^L $ are the best-fitted values to all the constants of expression in the $k$-th trial. Each column vector $C_{:,l}\in\mathbb{R}^K $ is the fitted values to the $l$-th constant in expression across all $K$ trials. The list of fitted expressions is stored in $\bm{\phi}=(\phi_1,\ldots,\phi_K)$.

Key information is obtained by examining the outcomes of multiple trial control variable experiments:
\begin{enumerate}
    \item Consistent close-to-zero fitness scores suggest that the fitted expression is close to the ground-truth equation in the reduced form.  
    \item Given that the equation is close to the ground truth, an open constant having similar best-fitted values across trials suggests that the open constants are \textit{stand-alone}. Otherwise, that open constant is a \textit{summary} constant, that corresponds to a sub-expression involving those control variables $\mathbf{x}_c$.
In other words, an open constant is standalone if the variance of its fitted values is small.
\end{enumerate}

\subsection{Vertical Symbolic Regression Framework} 
Algorithm~\ref{alg:cvsr} shows our vertical symbolic regressor framework. The high-level idea of \method is to construct increasingly complex symbolic expressions involving an increasing number of independent variables based on control variable experiments with fewer and fewer controlled variables.

To fit an expression of $m$ variables, initially, we control the values of all $m-1$ variables and allow only one variable to vary.
We would like to find a set of expressions $\{\phi_{1,1}, \ldots, \phi_{1,M}\}$, which best fit the data in this controlled experiment. 
Notice $\{\phi_{1,1}, \ldots, \phi_{1,M}\}$ are restricted to contain only a single free variable. 
We assume the availability of a symbolic regressor, which we denote as \texttt{Regressor}, to complete this task. Two implementations of the \texttt{Regressor} based on Genetic Programming (GP) and Monte-Carlo Tree Search (MCTS) will be discussed in Section \ref{sec:regressor}. 
Because the reduced hypothesis space involves only one independent variable, the fitting task is much easier than fitting the expressions involving all $m$ variables. 
Next, for each $\phi_{1,l}$, we examine the following: (1) whether the fitting errors are consistently small across all the trials. A small error 
implies that the equation found is close to the ground-truth formula reduced to the one free variable. 
We hence \textit{freeze} all the operators of the formula in this case. Freezing means that the \texttt{Regressor}
 in later rounds cannot change these operators. This step is denoted as \texttt{FreezeEquation} in Algorithm~\ref{alg:cvsr}.
(2) In the case of small fitting errors, 
we also inspect the best-fitted values of each open constant in the discovered equation across trials. 
The constant is probably a summary constant if its values vary across trials, i.e., the variance of the fitted constant across multiple trials is high.
In other words, these constants probably represent sub-expressions involving the controlled variables in the ground-truth equation.
We thus mark these constants as \textit{summary constants}. These constants will be expanded into sub-expressions in the upcoming rounds.
The remaining constants are probably classified as stand-alone. Therefore, we also freeze them.

In the second round, \method adds a second free variable and starts fitting $\{\phi_{2, 1}, \ldots, \phi_{2,M}\}$ using the data from 
 control variable experiments involving the 
 two free variables. 
 Similar to the previous step, all $\phi_{2,l}$ are restricted to only contain the two 
 free variables. 
 Moreover, the \texttt{Regressor} is initialized from the expressions of the first round $\{\phi_{1, 1}, \ldots, \phi_{1,M}\}$. 
 After obtaining the best-fit expressions from the \texttt{Regressor}, a similar inspection is performed for every candidate equation, examining the corresponding fitness scores and the variance of the fitted constants.
 This process repeats with an increasing number of variables involved. Eventually, the predicted expressions in the last round involve all $m$ variables.

\begin{algorithm}[!t]
   \caption{Vertical Symbolic Regression through Control Variable Experiments.}\label{alg:cvsr}  
   \begin{algorithmic}[1]
   \Require{ \#input variables $m$; Mathematical Operators $O_p$; \texttt{DataOracle} for the ground-expression; a symbolic regressor \texttt{Regressor}.}
   \Ensure{The best-predicted expression.}
   % \Statex \hspace{-1.6em}\textbf{Parameters:} \#multiple control variable trials $K$.
   \State $\mathbf{v}_c \gets \{x_1, \ldots, x_m\}$.\Comment{controlled variables}
   \State $\mathbf{v}_f \gets \emptyset$. \Comment{free variables}
    \State ${{D}_{global}} \gets \texttt{DataOracle}(\emptyset, \{x_1, \ldots, x_m\})$.
   \State $\mathcal{P} \gets \emptyset.$ \Comment{candidate expressions for round $i$}
   \State $\mathcal{Q}\gets \emptyset.$ \Comment{best expressions across all rounds}
   \medskip
   \For{$x_i \in \{x_1, \ldots, x_m\}$ }
        \State $\mathbf{v}_c \gets \mathbf{v}_c \setminus \{x_i\}$. \Comment{set $x_i$ to be free variable}
        \State $\mathbf{v}_f \gets \mathbf{v}_f \cup \{x_i\}$.
        \State ${{D}_o} \gets \texttt{DataOracle}(\mathbf{v}_c, \mathbf{v}_f)$. \Comment{construct the data oracle}
    
        \State $\mathcal{P} \gets \texttt{Regressor}(\mathcal{P}, D_o, O_p \cup \{\mathtt{const}, x_i\})$.
        
        \For{$\phi \in \mathcal{P}$} \Comment{multiple control variable trails}
        \State $\{T_k\}_{k=1}^K\gets\texttt{GenData}(D_o)$. \Comment{sample multiple batches of data}
        \State $\langle\mathbf{o}, C, \bm{\phi} \rangle\gets\texttt{CvExp}(\phi, \mathbf{v}_c, \mathbf{v}_f, \{T_k\}_{k=1}^K)$.  \Comment{fit multiple open constants in $\phi$}
            \State $\phi\gets $\texttt{FreezeEquation}($ \mathbf{o}, C, \bm{\phi}$).
        \EndFor
        %
        
        %\State \nj{The following are updated}
        \For{$\phi \in \mathcal{P}$} \Comment{update global best expression set}
        \State $T_k\gets\texttt{GenData}(D_{global})$. \Comment{sample one batch of data}
       \State  $\langle {o}, \mathbf{c},{\phi}\rangle\gets\texttt{CvExp}(\phi, \mathbf{v}_c, \mathbf{v}_f, T_k)$. \Comment{fit open constants in $\phi$}
    \State save ${\phi}$ into $\mathcal{Q}$ if ${\phi}$ is ranked high among all equations in $\mathcal{Q}$.
    \EndFor
   \EndFor
   \State \Return the equation with best fitness score in $\mathcal{Q}$.
\end{algorithmic}
\end{algorithm}

 The proposed \method framework needs (1) a symbolic regression algorithm (denoted as \texttt{Regressor}) that can efficiently find good candidate equations in the hypothesis space involving a subset of free variables $\mathbf{v}_f$.  In this work, we implement \texttt{Regressor} using genetic programming or Monte Carlo tree search, which are detailed in Section~\ref{sec:regressor}. We leave the adaptation of \method to other state-of-the-art symbolic regressors as future work. (2) A data oracle that allows us to query customized data points in which $\mathbf{x}_c$ is held at a constant value. The data oracle is discussed in section~\ref{sec:data-oracle}.

 The whole procedure of \method is shown in Algorithm~\ref{alg:cvsr}.  Here $\mathcal{P}$ is the set of expressions that the \texttt{Regressor} found in the current round. We assume that the ground-truth equation involves at most $m$ variables, $\{x_1, \ldots, x_m\}$. $x_1, \ldots, x_m$ are moved
 from the controlled to free variables in numerical order. 
 We agree that other orders may boost its performance. However, we leave the exploration of this direction as future work. 
 When a new variable $x_i$ becomes free, in line 10 \texttt{Regressor} is applied to extend the equations found in previous rounds (containing variables $x_1, \ldots, x_{i-1}$) to the ones that model the new control variable settings well ($x_1, \ldots, x_i$ are allowed to vary, $x_{i+1}, \ldots, x_m$ are held as constants).
 In line 9, we construct a new data oracle for the regressor to use because the set of controlled variables is updated. 
 The input arguments of the {\texttt{Regressor}} are as follows: 1) the set of equations found in the previous round,  2) a data oracle that generates data from the current control variable settings, and 3) the set of operators $O_p$ allowed -- when $x_i$ is moved from the set of controlled variables to free ones, we only allow changing the equations with all the arithmetic operators, the constant, and variable $x_i$.
 Finally, in Lines 11-14 of Algorithm~\ref{alg:cvsr}, \texttt{FreezeEquation} is called for every equation found in $\mathcal{P}$. For every equation $\phi$, we run $K$-trials of control variable experiments to determine whether each constant in the expression is summary or stand-alone. 
 Notice that this decision can be made by judging the variances of the fitted errors, according to Section \ref{sec:ctrl_exp}. 
 Because the \texttt{FreezeEquation} function is deeply connected to the \texttt{regressor} as well as the representations of the expressions, we will detail this function separately in the discussions of the integration with the GP algorithm (in section~\ref{sec:cv-gp}) and the MCTS algorithm (in section~\ref{sec:cv-mcts}).
 We also maintain $\mathcal{Q}$ as the set of best expressions across all rounds. 
 An equation found in an intermediate step, e.g., when certain variables are held constant, can still be a good candidate equation for modeling all the independent variables. 
 We maintain a separate data oracle $D_{global}$, in which no variables are controlled. 
 For every equation in $\mathcal{P}$, 
 we fit its open constants with data drawn from $D_{global}$. 
 The equation will be updated into $\mathcal{Q}$ if it is ranked higher among the best equations found thus far.
In this case, when the algorithm arrives at the last round, $\mathcal{Q}$ will contain the set of best expressions.

Fig.~\ref{fig:running-example} shows the high-level idea of fitting an equation using the \method. 
Here, the process has four stages, each stage with a decreased number of controlled variables. The trial data in each stage are shown at the bottom, and the best expression found is shown at the top. The summary constants are boldfaced and colored blue. The readers can see how the fitted equations grow into the final ground-truth equation, with one free variable added at a time.

\subsection{The Availability of Data Oracle} \label{sec:data-oracle}
A crucial assumption behind the success of \method is the availability of a  $\mathtt{DataOracle}$ that returns a (noisy) observation of the dependent output with input 
variables in $\mathbf{v}_c$ controlled and $\mathbf{v}_f$ free. 
This differs from the classical setting of symbolic regression, where a dataset is obtained prior to learning \citep{matsubara2022rethinking,doi:10.1080/15598608.2007.10411855}. 
Such a data oracle represents conducting control variable experiments in the real world, which can be expensive. The detailed construction of the data Oracle in this study is provided in Appendix~\ref{apx:data_oracle}.

We argue that the integration of experiment design in the discovery of scientific knowledge is indeed the main driver of the successes of \method. 
The idea of building AI agents that mimic human scientific discovery has achieved tremendous success in early works~\citep{DBLP:conf/ijcai/Langley77,DBLP:conf/ijcai/Langley79,DBLP:conf/ijcai/LangleyBS81}. Recent work~\citep{chen2022generalisation,keren2023computational,DBLP:conf/gecco/HautBP22,DBLP:conf/gecco/HautPB23} also pointed out the importance of having a data oracle that can actively query data points, rather than learning from a fixed dataset.
Our work does not intend to show that VSR is superior in every use case. 
We acknowledge that fully controlled experiments may be difficult and expensive to conduct in many scenarios.  
In cases where it is difficult to obtain such a data oracle, one possible solution is to use deep neural networks to learn a data generator for the given set of controlled variables.
We leave it as future work. 
We also would like to point out that following the vertical discovery paths can improve any symbolic regression method. In this work, GP and MCTS are used as demonstrating examples. 

When benchmarking all the algorithms, we ensure that every algorithm for comparison has no access to information about the ground-truth equation other than through querying the data oracle. The baseline training and testing settings are described in section~\ref{sec:experiment-test-setting}.

%% file: tex/4.regressor.tex
\section{Symbolic Regressor} \label{sec:regressor}
Our vertical symbolic regression framework as outlined in Algorithm~\ref{alg:cvsr} can conceptually adopt any symbolic regression algorithm as the \texttt{Regressor} (line $10$ in Algorithm~\ref{alg:cvsr}). 
This section describes the integration of currently popular symbolic regression algorithms -- Genetic Programming (GP) and Monte Carlo Tree Search (MCTS) -- into the vertical symbolic regression framework.
% to discover complex expressions involving multiple input variables. 
We will describe the classic setups for GP and MCTS and then their needed modifications to fit in \method.
The modified regression algorithms are named \cvgp and \cvmt and are outlined in Algorithms~\ref{alg:cvgp} and~\ref{alg:mcts}, respectively.

\subsection{Vertical Symbolic Regression via Genetic Programming} \label{sec:cv-gp}
% \nasim{restart here}
In this section, we will describe how we can adopt Genetic Programming as the \texttt{Regressor} for vertical symbolic regression. We first introduce a suitable tree representation of symbolic expression that we use in GP, then describe the classic GP approach for symbolic regression, and finally \cvgp, which is our modified version of GP.

\begin{figure}[!t]
    \centering
    \includegraphics[width=\linewidth]{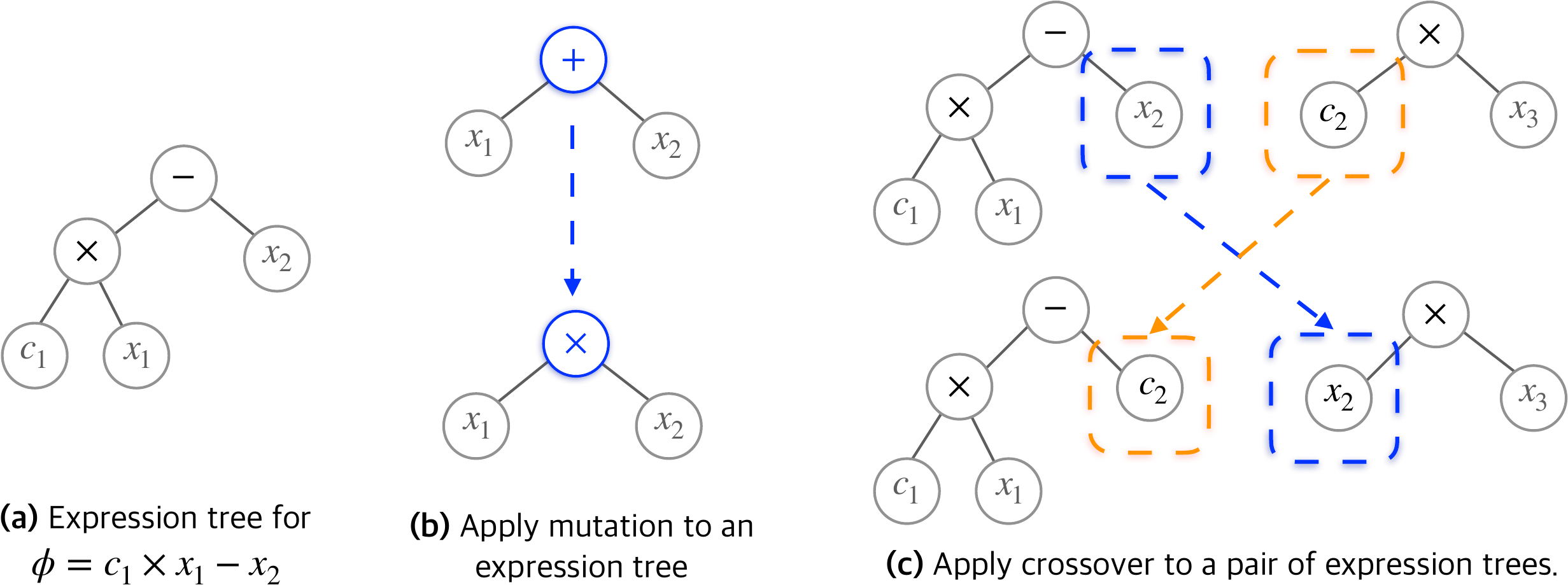}
    \caption{The mutation and crossover of expression trees in the GP algorithm. \textbf{(a)} Expression $\phi=c_1\times x_1-x_2$ represented as a binary tree. The leaf nodes are either variables like $x_1, x_2$ or constants like $c_1$. The internal nodes are mathematical operators like $\times, -$. \textbf{(b)} The mutation operation randomly modifies a node in the expression tree. \textbf{(c)} The crossover operation randomly picks a pair of sub-trees in two expressions and swaps the sub-trees.}
    \label{fig:mutate-mate}
\end{figure}

\subsubsection{Symbolic Expression as Tree}
A symbolic expression can be represented as an \textit{expression tree}, where variables and constants correspond to leaves, and operators correspond to the inner nodes of the tree. 
An inner node can have one or multiple child nodes depending on the arity of the associated operator. For example, a node representing the addition operation ($+$) has 2 children, whereas a node representing trigonometric functions like $cosine$ operation has a single child node. 
The inorder traversal of the expression tree gives out the expression in symbolic form. Fig~\ref{fig:mutate-mate}(a) presents an example of such an expression tree.

\subsubsection{Genetic Programming for Symbolic Regression}
Genetic Programming (GP) \citep{koza1994genetic} has been a popular randomized algorithm for symbolic regression. 
The core idea of GP is to first create a pool of random symbolic expressions represented as \textit{expression trees}, and then iteratively improve this pool according to a user-defined metric i.e., fitness score.
The fitness score of a candidate expression measures how well the expression fits a given dataset. 
Each iteration/generation of GP consists of 3 basic operations -- \textit{selection, mutation} and \textit{crossover}. In the \textit{selection} step, candidate expressions with the highest fitness scores are retained in the pool, while those with the lowest fitness scores are discarded. 
In the \textit{mutation} step, sub-expressions of some randomly selected candidate expressions are altered with some probability, and in the \textit{crossover} step, the sub-expressions of different candidate expressions are interchanged with some probability. From the implementation perspective, \textit{mutation} changes a node of the expression tree while \textit{crossover} is the exchange of subtrees between a pair of trees.
We show a visual representation of \textit{mutation} and \textit{crossover} in Fig.~\ref{fig:mutate-mate}.
A few randomly chosen expressions are also added to the pool for diversity, and the next iteration of GP repeats this whole process.  
After the final generation, we obtain a pool of expressions with high fitness scores, i.e., expressions that fit the data well, as our final solutions. 

\subsubsection{Adaptation of Genetic Programming to Vertical Symbolic Regression}
The \cvgp is a minimally modified genetic programming algorithm that can be used as \texttt{Regressor} in Algorithm~\ref{alg:cvsr} in our control variable setup.

\begin{algorithm}[!t]
   \caption{ \texttt{\cvgp} as \texttt{Regressor} Algorithm~\ref{alg:cvsr}.}
   \label{alg:cvgp}
   \textbf{Input:} Candidate expression pool $\mathcal{P}$; 
   data Oracle under controlled variable $D_o$; 
   allowed mathematical operators $O_p$.\\
   \textbf{Parameters:} number of generations, probabilities of mutation and crossover, and size of the GP pool $M$.\\
   \textbf{Output:} Best predicted expression pool.
   \begin{algorithmic}[1]
   \For{each generation}
   \State Calculate the fitness scores of the expressions in $\mathcal{P}$.
   \State $\mathcal{P} \gets Selection(\mathcal{P})$.
   \State $\mathcal{P} \gets Mutation(\mathcal{P},O_p)$.
   \State $\mathcal{P} \gets Crossover(\mathcal{P})$.
   \State $\mathcal{P} \gets Optimize(\mathcal{P},D_o)$.
\EndFor
  
   \State $\mathcal{P}\gets$ Top $M$ expressions from $\mathcal{P}$ according to fitness score.
   
 \State \Return  $\mathcal{P}$.
\Statex \hrulefill

\Statex \textbf{Helper procedures} for the {\cvgp}.
    \Procedure{$Selection$}{$\mathcal{P}$}
        \State Rank and return the topmost expressions from $\mathcal{P}$ according to fitness scores.
    \EndProcedure
    \Procedure{$Mutation$}{$\mathcal{P},O_p$}
        \For{every expression $\phi$ in $\mathcal{P}$}
        \State Select $\phi$ for mutation with some probability.
        \State If selected, randomly change a subexpression of $\phi$ with operators in $O_p$.
        \EndFor
    \EndProcedure
    \Procedure{$Crossover$}{$\mathcal{P}$}
        \State For every pair of expressions in $\mathcal{P}$, randomly exchange subexpressions to create two new expressions.
    \EndProcedure
    \Procedure{$Optimize$}{$\mathcal{P},D_o$}
        \For{every expression $\phi$ in $\mathcal{P}$}
        \State $T_k \gets$ Sample data with Oracle $D_o$.
        \State Find optimal constants in $\phi$ to fit data $T_k$ using an off-the-shelf optimizer.
        \EndFor
    \EndProcedure
\end{algorithmic}
\end{algorithm}

Algorithm~\ref{alg:cvgp} gives the general outline of \cvgp. Line $1-8$ gives the basic outline of the algorithm, while line $9-20$ gives out details of the subroutines. We start with a given pool of candidate symbolic expressions, a data Oracle, and a library of mathematical operators as inputs. The total number of generations to run, the probability of mutation and crossover, and the size of the candidate expression pool are internal parameters and are set manually.
% At the initialization step, we count the size of the hall-of-fame expression set (line $1$). 
In every generation, we first calculate the fitness scores of the candidate expressions according to the fitness score function and retain the high-scoring expressions through the \textit{Selection} step (line $2-3$). We then perform random mutation and crossover of the candidate expressions (line $4-5$). 
Using the data Oracle under the control variable setup, we then sample the dataset and find the optimal values of the constants in the new candidate expressions (line $6$). 
At the end of the final generation, we retain the top $M$ candidate expressions according to fitness scores, and return this pool $\mathcal{P}$ (line $7-8$). 

There are two key differences between the classic GP and our \cvgp:
\begin{enumerate}
    \item During $Mutation$ and $Crossover$, our \cvgp algorithm only alters the mutable nodes of the candidate expression trees. In classic GP, all the tree nodes are mutable, while in \cvgp, the mutable nodes of the expression trees and set of operators $O_p$ are preset by the \texttt{FreezeEquation} in Algorithm~\ref{alg:cvsr}. 
    \item The $Optimize$ function in \cvgp dynamically samples data with Oracle $D_o$ under the control variable setup, whereas classic GP uses a fixed static dataset.
\end{enumerate}

\subsection{Vertical Symbolic Regression via Monte Carlo Tree Search} \label{sec:cv-mcts} 
In this section, we describe how we can modify the Monte Carlo Tree Search (MCTS) for vertical symbolic expression. We will start by describing how we represent expression in a suitable manner for MCTS, and then describe the key idea of MCTS for symbolic regression. Afterward, we will show \cvmt, which is our modified version of the MCTS for vertical symbolic regression.

\subsubsection{Symbolic Expression with Context-Free Grammar}
A symbolic expression can be represented using an appropriate context-free grammar~\citep{DBLP:conf/icml/TodorovskiD97}. A context-free grammar is represented by a tuple of 4 elements $\mathcal{G} = (V,\Sigma, R, S)$, where $V$ is a set of non-terminal symbols, $\Sigma$ is a set of terminal symbols, $R$ is a set of production rules and $S\in V$ is a start symbol. 
In our CFG for symbolic expression, we use:
\begin{itemize}
    \item Set of non-terminal symbols representing sub-expressions as $V=\{A\}$.
    \item Set of input variables and constants $\{x_1,x_2,\dots,\mathtt{const}\}$ as $\Sigma$.
    \item Set of production rules representing possible mathematical operations such as addition, subtraction, multiplication, and division, as $R$.
    \item A single start symbol $A \in V$.
\end{itemize}

Beginning with the start symbol $A$, successive applications of the production rules in $R$ in different orders result in different CFG expressions. A CFG expression with only terminal symbols is a valid mathematical expression, whereas expressions with a mixture of non-terminal and terminal symbols can be further converted into other CFG expressions. 

In Fig.~\ref{fig:expr-as-rules}, we present an example of how to generate the mathematical expression $\phi= c_1\times x_1\div x_2$ from the start symbol $A$ using the CFG production rules $R$.
We first use the multiplication production rule $A\to A\times A$. Here, $\to$ represents replacement. Using the rule $A\to A\times A$ means that the symbol $A$ in expression $\phi=A$ is replaced with $A\times A$, resulting in $\phi = A\times A$. By using production rules repeatedly to replace non-terminal symbols, we finally arrive at our desired mathematical expression $\phi= c_1\times x_1\div x_2$.

\begin{figure}[!t]
    \centering
    \includegraphics[width=.9\linewidth]{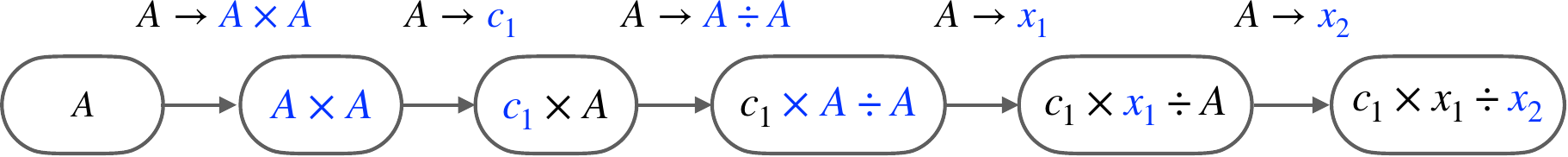}
    \caption{Symbolic expression $\phi=c_1\times x_1\div x_2$ represented as a sequence of CFG production rules. Each rule is applied step by step and we highlight the parts (in blue) that get replaced.}
    \label{fig:expr-as-rules}
\end{figure}

\subsubsection{Monte Carlo Tree Search for Symbolic Regression}

\mcts for symbolic regression is a systematic search process, that involves balancing between exploration and exploitation while searching for the optimal solution. \mcts maintain a \textit{search tree}, where \textit{nodes} represent expressions according to a context-free grammar $\mathcal{G}$, and an \textit{edge} represents a production rule of $\mathcal{G}$~\citep{DBLP:conf/icml/TodorovskiD97,DBLP:conf/dis/GanzertGSK10,DBLP:journals/kbs/BrenceTD21}. The node expressions can contain both terminal and non-terminal symbols of $\mathcal{G}$. A node expression containing at least one non-terminal symbol is considered ``expandable", in the sense that the search tree can be expanded from this node, by using different production rules of $\mathcal{G}$ and creating child node expressions.

Each node in the search tree also maintains an associated upper confidence bound (UCB) score \citep{kocsis2006bandit} as follows: 
\begin{equation*}
    UCB(s, a) = Reward(s, a) + c\sqrt{\ln[N (s)]/N (s, a)} 
\end{equation*}
Here, $Reward(s,a)$ denotes the averaged reward after applying the production rule $a\in R$ at search tree node $s$; $N(s)$ is the number of visits to node $s$, while $N(s,a)$ is the number of times rules $a$ is selected at node $s$. The constant $c$ is a hyper-parameter.

\begin{figure}[!t]
    \centering
    \includegraphics[width=1.0\linewidth]{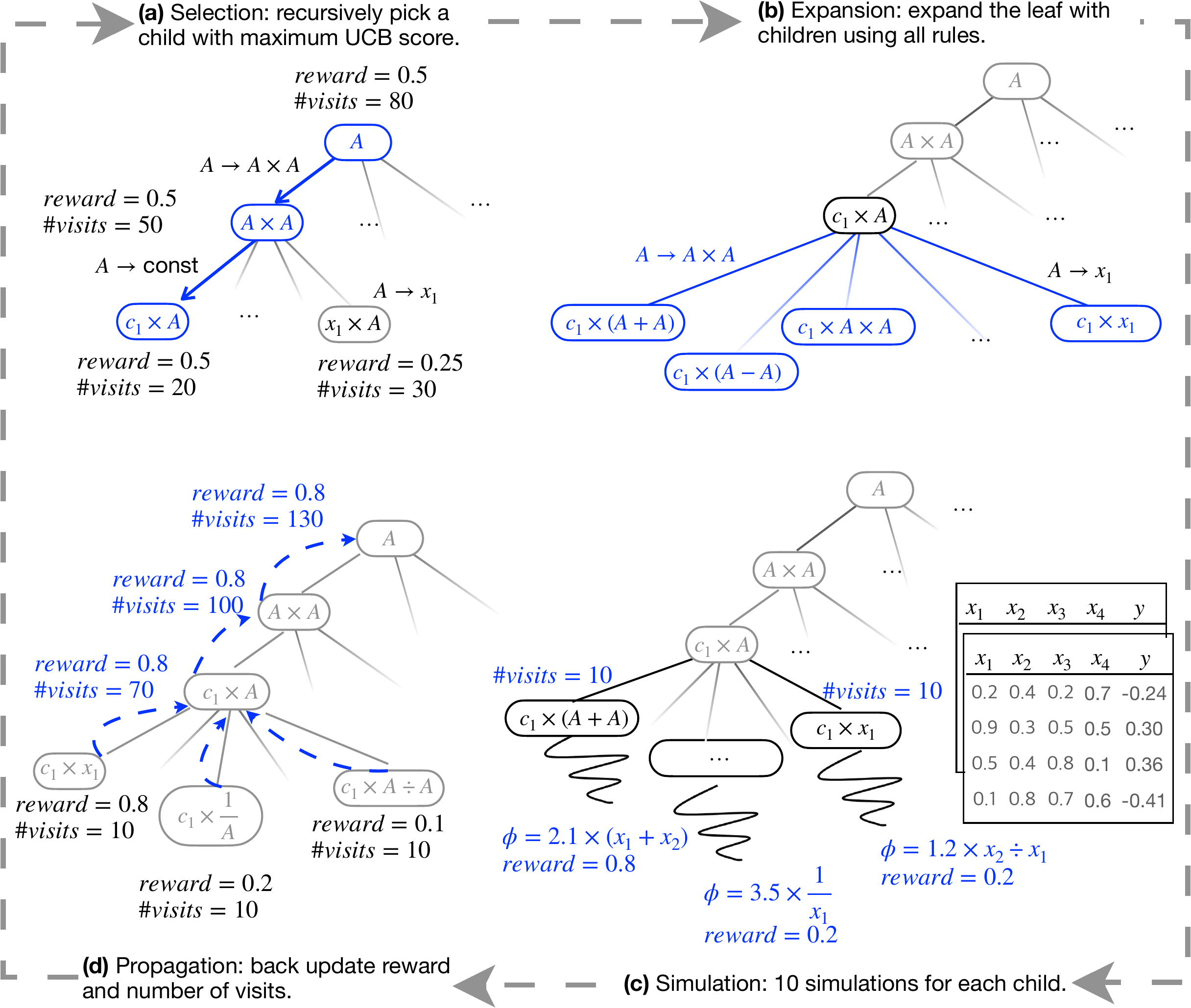}
    \caption{Phases of the Monte Carlo tree search algorithm.
    A search tree is grown through repeated application of the above four phases. \textbf{(a)} Starting from the root, we recursively pick the child with the highest upper confidence bound score among its children until reaching a leaf.  \textbf{(b)} We expand the selected leaf with all the rules and generate its children. \textbf{(c)} For each child node, we run 10 simulations to fill out the non-terminal symbols using the existing rules to form a complete expression. Then, we use the data to fit the values of the open constants in the randomly generated expressions. \textbf{(d)} We populate the rewards as well as the number of visits of expanded children to their parents up to the root of the search.}
    \label{fig:mcts-steps}
\end{figure}

To discover an optimal mathematical expression to fit a dataset, MCTS starts with an initial guess expression containing non-terminal and (optionally) terminal symbols from context-free grammar $\mathcal{G}$. This initial expression acts as the root node of the search tree. Afterward, MCTS applies the following 4 operations for a fixed number of iterations/episodes:
\begin{itemize}
    \item Selection: Starting from the root node, successively select children with the best UCB scores to obtain a leaf node $s$ of the search tree   
    \item Expansion: At leaf node $s$, apply the production rules from $\mathcal{G}$, creating child nodes of $s$, and expanding the search tree by one level from node $s$.
    \item Simulation: For every child of $s$, perform a fixed number of rollout rounds. The UCB scores for each child are calculated by aggregating the results of these rollouts. 
    In each round of rollouts, generate a valid mathematical expression $\phi$ from the child node by randomly applying the production rules of $\mathcal{G}$. After generating $\phi$, find the optimal constant values and fitness score of $\phi$ by using the dataset. The fitness score acts as a reward for subsequent computations.
    \item Backpropagation: Update the UCB scores and the number of visits of $s$, and all ancestors of $s$ up to the root node of the search tree. 
\end{itemize}

At the end of the final iteration, the expression with the highest fitness score, generated during the numerous rounds of rollouts during simulation is finalized as the best expression.

\subsubsection{Adaptation of Monte Carlo Tree Search to Vertical Symbolic Regression}
The \cvmt is a modified version of the classic MCTS for symbolic regression, that can be used as a \texttt{Regressor} in our vertical symbolic regression framework.

% nasim
\begin{algorithm}[!t]

   \caption{\cvmt as \texttt{Regressor} in Algorithm \ref{alg:cvsr}. 
   % \nasim{why we need hall of fame} 
   }
   \label{alg:mcts} 
   \textbf{Input:} Initial expression $\{\phi_{init}\}$; 
   data Oracle under controlled variable $D_o$; 
   set of mathematical operators $O_p$.\\
   \textbf{Output:} best candidate expression $\{\phi^*$\}.\\
   \textbf{Parameters:} Total episodes $N_{episodes}$; Number of simulations $N_{sim}$.
   \begin{algorithmic}[1]
   \State $\mathcal{G} \gets$ Construct a context-free grammar from operators $O_p$.
   \State $root\gets$ Create node expression from $\phi_{init}$, by replacing summary constants with non-terminal symbols in $\mathcal{G}$.
   \For{\textit{every episode}}
        \State $current \gets SelectBestLeaf(root)$ .\Comment{Selection}
        \State $children(current) \gets Expand(current)$. \Comment{Expansion}
        \For{every $s\in children(current)$} \Comment{Simulation}
            \State $Q \gets Rollout(s,N_{sim})$.
        \For{every $\phi \in Q$}
        \State $T_k \gets$ Sample data with Oracle $D_o$.
        \State Fit constants in $\phi$ to fit data $T_k$ using an off-the-shelf optimizer.
        \State compute UCB score for $s$. 
   \EndFor
   \State Update UCB score for $current$ node and its ancestors. \Comment{Backpropagation}
   \EndFor
   \EndFor
   \State $\phi^*\gets$ select expression with the best fitness score among all searched expressions.

   \State \Return $\{\phi^*\}$. 
   \Statex
   \hrulefill
   \Statex \textbf{Helper procedures} for \cvmt
   \Procedure{$SelectBestLeaf$}{$root$}
        \State Starting from $root$, repeatedly select the child node with the best UCB score until it reaches a leaf node.
    \EndProcedure
    \Procedure{$Expand$}{$current$}
        \State Apply each production rule of $\mathcal{G}$ on $current$, to obtain child nodes of $current$ node.
    \EndProcedure
    \Procedure{$Rollout$}{$s,N_{sim}$}
        \State $Q\gets \emptyset$.
        \For{$N_{sim}$ times}
        \State $\phi \gets$ Apply production rules of $\mathcal{G}$ randomly and generate expression containing only terminal symbols of $\mathcal{G}$.
        \State Add $\phi$ into $Q$.
        \EndFor
        \Return Set of randomly generated expressions $Q$.
    \EndProcedure
\end{algorithmic}
\end{algorithm}

% \nasim{wrong rule in figure text, should be A-> A+A, not multiplication} 
\begin{figure}[!t]
    \centering
    \includegraphics[width=\linewidth]{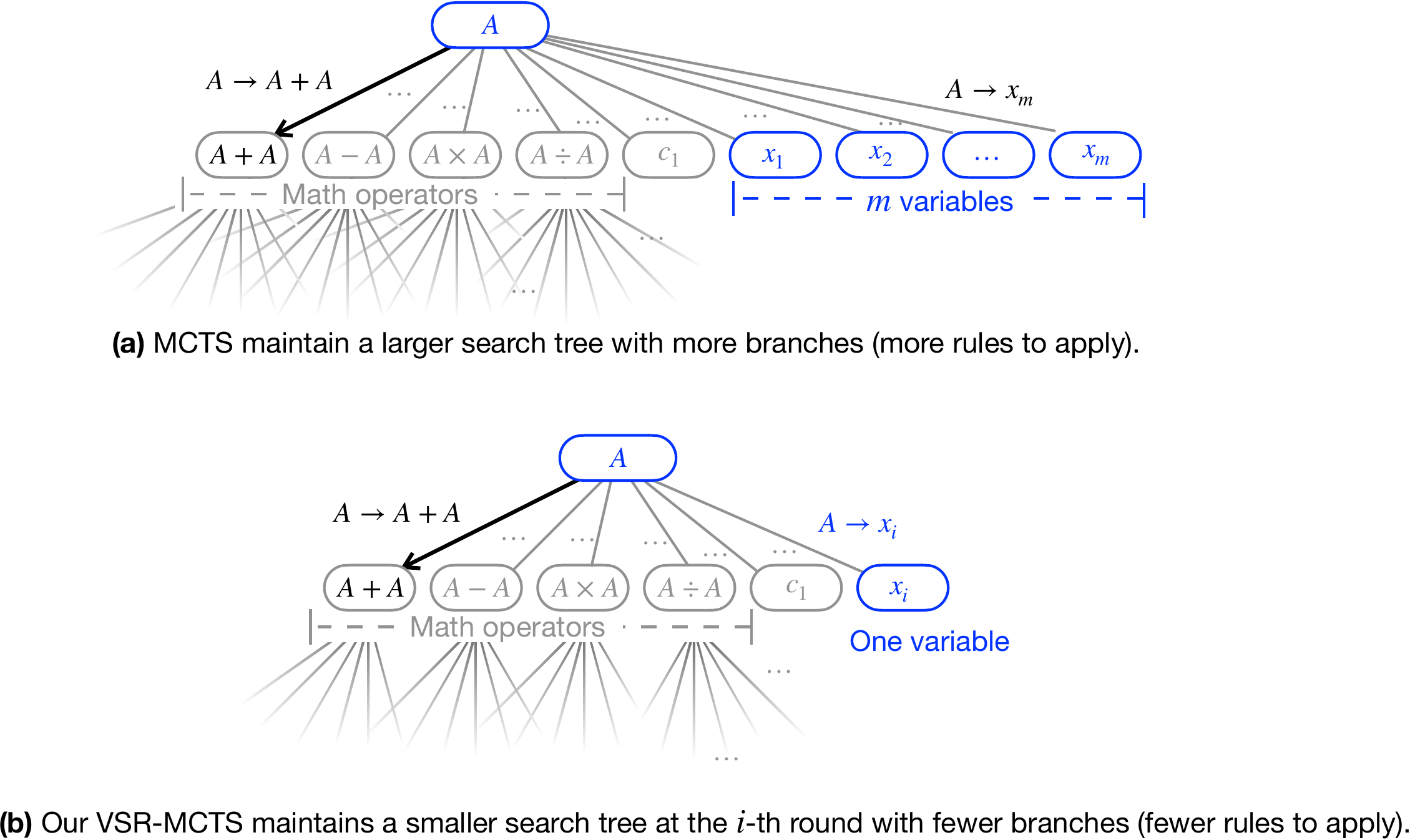}
    \caption{\cvmt maintains a smaller search tree compared to \mcts, and thus is more efficient to explore. \textbf{(a)} \mcts maintains a search tree that supports all the mathematical operators and all $m$ input variables, to search for the full expression directly. \textbf{(b)} \cvmt maintains a smaller search tree with the same set of mathematical operators but one variable (variable $x_i$ in the $i$-th round), to search for the reduced form expression, which has a much greater chance than \mcts.}
    \label{fig:cvmt-diff}
\end{figure}

Algorithm~\ref{alg:mcts} gives the general outline of \cvmt. We start with the current best symbolic expression $\phi_{init}$, a data Oracle under control variable setup, and a library of mathematical operators $O_p$ as input. 
First, we construct a context-free grammar $\mathcal{G}$ for symbolic regression using $O_p$ (line 1). For example, a simple CFG can have a single non-terminal symbol $A$. Mathematical operations in $O_p$ such as addition are added as a production rule $A\to A+A$ in $\mathcal{G}$. For each variable $x_i\in O_p$, we add the production rule $A\to x_i$, and for the constant coefficient, we add the rule $A\to \texttt{const}$.
We then convert $\phi_{init}$ by replacing each summary constant with a non-terminal symbol. This constitutes the root node of our search tree (line 2). We then repeat 4 basic operations of \mcts -- \textit{selection, expansion, simulation} and \textit{backpropagation} for a fixed number of episodes. In each episode, first, we select the best \textit{current} leaf node of the search tree, by starting from the root node and repeatedly selecting the child with the best UCB score (line 4). We then apply the production rules of $\mathcal{G}$, and obtain the child nodes of the \textit{current} node. For every child of \textit{current} node, we then perform $N_{sim}$ rounds of rollouts and compute their UCB score (lines 6-11). In each rollout round, we generate a valid mathematical expression from the child nodes, by randomly applying the production rules of $\mathcal{G}$, until all non-terminal symbols are eliminated. Using the data Oracle under the control variable setup $D_o$, we find the optimal constant coefficients for the expressions from simulation rollout rounds and compute their reward scores/fitness scores. We can now calculate the UCB score of every child of $current$, and subsequently backpropagate these results and update the UCB scores of $current$ node and all its ancestors (line 12). At the end of the final episode, we return the expression with the optimal fitness score among all the generated expressions as a single element set (line 13).

%The key difference between our \cvmt and classic MCTS for symbolic regression is that 

There are two key differences between classic \mcts and our \cvmt:
\begin{enumerate}
\item The root of \cvmt is transformed from the best expression in the previous round in \cvmt, while it is always ``$A$'' in \mcts. During the transfer, summary constants are replaced with the non-terminal symbol (i.e., ``$A$''). The production rules in \cvmt at $i$-th round consider only variable $x_i$ ($A\to x_i$), while \mcts considers the rules of all the variables (i.e., $A\to x_1,\cdots,\ldots, A\to x_m$). See Fig.~\ref{fig:cvmt-diff} for a visual explanation. 
    \item We use our data Oracle under the control variable setup $D_o$ to find the optimal constant coefficients of the mathematical expressions. Classic MCTS uses a fixed static dataset for this step.
\end{enumerate}

%% file: tex/6.theory.tex
\section{Theoretical Analysis} 
We show in this section that 
vertical symbolic regression brings an exponential reduction in the search space when fitting a particular class of symbolic expressions. 
To see this, we assume that symbolic regression algorithms follow a search order from simple to complex symbolic expressions and the data is noiseless. Since the following analysis is orthogonal to the representations of the expressions, we use the tree format to represent the expressions. 
\medskip
\begin{definition}
Define the set of binary expression trees containing exactly $l$ nodes as $T(l)$, the size of which is denoted as $|T(l)|$.
\end{definition}

\medskip
\begin{definition}
Define the hypothesis space of expression trees containing at most $l$ nodes as the set of all symbolic expression trees involving at most $l$ nodes, which is denoted as $S(l)$, and the size of the hypothesis space is $|S(l)|$.
\end{definition}
Notice that different expression trees may correspond to the same symbolic expression in the mathematical definition. Directly counting different expressions to bound the size of the hypothesis space is intractable.  For simplicity, we thus count the number of different expression trees instead of the number of unique expressions.

\medskip
\begin{lemma}\label{lem:searchspace}
For simplicity, assume that all operators are binary, and let $o$ be the number of operators and $m$ be the number of input variables. The size of the hypothesis space of symbolic expression trees of $l$ nodes scales exponentially; more precisely at $\mathcal{O}((4(m+1)o)^{\frac{l-1}{2}})$ and $\Omega((4(m+1)o)^{\frac{l-1}{4}})$. 
\end{lemma} 
\begin{proof}
    Assuming all operands are binary, a symbolic expression tree containing $l$ nodes has $\frac{l+1}{2}$ leaves and $\frac{l-1}{2}$ internal nodes. 
    The number of binary trees of $\frac{l-1}{2}$ internal nodes is given by the Cantalan number $C_{{(l-1)}/{2}} = \frac{l-1}{{l+1}}\binom{l-1}{{(l-1)}/{2}}$, 
    which asymptotically scales at $\frac{2^{{l-1}}}{{\left(\frac{l-1}{2}\right)}^{\frac{3}{2}}\sqrt{\pi}}$. 
    A symbolic expression replaces each internal node of a binary tree with an operand and replaces each leaf with either a constant or one of the input variables. 
    Because there are $o$ operands and $m$ input variables, the total number of different symbolic expression trees involving $l$ nodes is given by:
    \begin{align*}
    |T(l)| = C_{\frac{(l-1)}{2}} (m+1)^{\frac{l+1}{2}} o^\frac{l-1}{2} &=\frac{{l-1}}{{l+1}} \binom{l-1}{\frac{(l-1)}{2}}    
    \sim \frac{(4(m+1)o)^{\frac{l-1}{2}}}{ \left(\frac{l-1}{2}\right)^{\frac{3}{2}}}.
    \end{align*}
    Hence, the total number of trees up to $l$ nodes is:
    \begin{align*}
    |S(l)| = \sum_{i=0}^{\frac{(l-1)}{2}} T(2i+1) \sim \sum_{i=0}^{\frac{(l-1)}{2}} \frac{(4(m+1)o)^{i}}{ i^{\frac{3}{2}}}.
    \end{align*}
    When $i$ is sufficiently large, we can approximate the right-hand side term by:
    \begin{align*}
    (4(m+1)o)^{\frac{i}{2}} \leq \frac{(4(m+1)o)^{i}}{ i^{\frac{3}{2}}} \leq (4(m+1)o)^{i}.
    \end{align*}
    Therefore, the upper bound of the size of the  hypothesis space can be obtained:
    \begin{equation*}
    |S(l)| \leq \sum_{i=0}^{\frac{(l-1)}{2}} (4(m+1)o)^{i} \in \mathcal{O}\left((4(m+1)o)^{\frac{(l-1)}{2}}\right).
    \end{equation*}
    Similarly, we can obtain the lower bound for the size of the hypothesis space as follows:
    \begin{equation*}
    |S(l)| \geq\frac{ (4(m+1)o)^{\frac{(l-1)}{2}} }{ (\frac{(l-1)}{2})^{\frac{3}{2}}} \geq (4(m+1)o)^{\frac{(l-1)}{4}},
    \end{equation*}
    which implies 
    \begin{equation*}
    |S(l)| \in \Omega\left((4(m+1)o)^{\frac{l-1}{4}}\right).
    \end{equation*}
The proof is complete.
\end{proof}

The proof of Lemma~\ref{lem:searchspace} mainly involves counting binary trees. The exact mathematical formula is not important. 
For our purposes, it is sufficient to know that the size is exponential in the size of the expression tree $l$. 

\medskip
\begin{definition}[Simple to complex search order]
A symbolic regression algorithm follows a simple to complex search order if it expands its hypothesis space from short to long symbolic expressions; \textit{i.e.}, first searches for the best symbolic expressions in $S(1)$, and then in $S(2) \setminus S(1)$, etc. 
\end{definition}

In general, it is difficult to quantify the search order of any symbolic regression algorithm. 
However, we believe that the simple to complex order reflects the search procedures of a large class of symbolic regression algorithms. %, including our \method. 
In fact, \cite{DBLP:journals/tcyb/ChenXZ22a} explicitly use regularizers to promote the search of 
simple and short expressions. 
Our \method follows the simple to complex search order approximately.
Indeed, GP or MTCS may encounter more complex equations
before their simpler counterparts. 
However, in general, the expressions are built from simple to complex equations in the algorithms we proposed. % by mating and mutating operations in genetic programming algorithms. 

\medskip
\begin{proposition}[Exponential Reduction in the Hypothesis Space] There exists a symbolic expression $\phi$ 
of $(4m-1)$ nodes, and a horizontal symbolic regression algorithm following the simple to complex 
 order has to explore a hypothesis space whose size is exponential in $m$ to find the expression, while 
\method following the simple to complex 
 order only expands $\mathcal{O}(m)$ constant-sized hypothesis spaces.
 
\begin{proof}
Consider a dataset generated by the ground-truth symbolic expression made up of 2 operators ($+, \times$), $2m$ input variables, and $(4m-1)$ nodes:
\begin{equation}
(x_1 + x_2) (x_3 + x_4) \ldots (x_{2m-1} + x_{2m}).
\end{equation}
To search for this symbolic regression, a horizontal symbolic regression algorithm following the simple to complex order needs to consider all 
expression trees up to $(4m-1)$ nodes. According to Lemma~\ref{lem:searchspace}, the normal algorithm has a hypothesis space of at least $\Omega((16m + 8)^{m-1/2})$, which is exponential in $m$.

On the other hand, {in the first step of \method}, $x_{2}, \ldots, x_{2m}$ are controlled and only $x_1$ {is free}.
In this case, the ground-truth equation in the reduced form is
\begin{equation}
(x_1 + C_1) D_1,
\end{equation}
in which both $C_1$ and $D_1$ are summary constants. Here $C_1$ represents $x_2$ and $D_1$ represents $(x_3 + x_4) \ldots (x_{2m-1} + x_{2m})$ in the control variable experiments. 
The reduced equation is quite simple under
the controlled environment. \method should be able to find 
the ground-truth expression exploring hypothesis space $S(5)$. 

Proving using induction. In step $2i~(1\leq i \leq m)$, variables $x_{2i+1}, x_{2i+2}, \ldots, x_{2m}$ are held as constants, $x_1, \ldots, x_{2i}$ are allowed to vary. 
The ground-truth expression in the reduced form found in the previous $(2i-1)$-th step is:
\begin{equation}\label{eq:true2i-1}
(x_1 + x_2) \ldots (x_{2i-1} + C_{2i-1}) D_{2i-1}.
\end{equation}
The \method needs to extend this equation to be the ground-truth expression in the reduced form for the $2i$-th step, which is:
\begin{equation}\label{eq:true2i}
(x_1 + x_2) \ldots (x_{2i-1} + x_{2i}) D_{2i}.
\end{equation}
The change is to replace the summary constant $C_{2i-1}$ to $x_{2i}$.
Assume that the data is noiseless and that \method can confirm expression \eqref{eq:true2i-1} is the ground-truth reduced-form expression for the previous step. This means that all the operators and variables will be frozen by the \method, and only $C_{2i-1}$ and $D_{2i-1}$ are allowed to be replaced by new expressions. 
Assume the \method algorithm follows the simple to complex search order, it should find the ground-truth expression~\eqref{eq:true2i} by searching replacement expressions of lengths up to $1$.

Similarly, in step $2i+1$, assume that \method confirms the ground-truth expression 
in the reduced form in step $2i$, \method also only needs to search in constant-sized spaces to
find the new ground-truth expression. Overall, we can see that only $\mathcal{O}(m)$ searches in constant-sized spaces are needed for \method to find the final ground-truth expression. 
\end{proof}
\end{proposition}

%% file: tex/7.related.tex
\section{Related Work}

\textbf{AI-driven Scientific Discovery.} 
Recently AI has been highlighted to enable scientific discoveries in diverse domains~\citep{doi:10.1126/science.abj6511,jumper2021highly,sholl2022density,wang2023scientific}. 
Early work in this domain focuses on learning logic (symbolic) representations~\citep{BRADLEY2001reasoning,Bridewell2008inductive}.
Recently, learning Partial Differential Equations (PDEs) from data has also been studied extensively 
~\citep{Dzeroski1995lagrange,brunton2016sparse,PhysRevE.100.033311,doi:10.1098/rspa.2018.0305,iten2020discovering,DBLP:conf/nips/CranmerSBXCSH20,Raissi20Fluid,RAISSI2019PhysicsInformedNN,Liu21AIPoincare,nanovoid_tracking,chen2018neural}. 
In this domain, a line of works develops robots that automatically refine the hypothesis space, some with human interactions~\citep{langey1988scientificdiscovery,Valdes1994,king2004functional,king2009autosci}.
These works are quite related to ours because they also actively probe the hypothesis spaces, albeit they are in biology and chemistry. 

\medskip
\textbf{Symbolic Regression.} Symbolic regression is proven to be NP-hard~\citep{journal/tmlr/virgolin2022}, because the search space of all possible symbolic expressions is exponential with respect to the number of input variables.  
Early works in this domain are based on heuristic search~\citep{LANGLEY1981DataDiscovery,LENAT1977ubiquity}.
Genetic programming turns out to be effective in searching for good candidates of symbolic expressions~\citep{journal/2020/aifrynman,DBLP:conf/gecco/VirgolinAB19,DBLP:conf/gecco/HeLYLW22}. Reinforcement learning-based methods propose a risk-seeking policy gradient to find the expressions~\citep{DBLP:conf/iclr/PetersenLMSKK21,DBLP:conf/nips/MundhenkLGSFP21}. 
Other works use RL to adjust the probabilities of genetic operations~\citep{DBLP:journals/apin/ChenWG20}. Additionally, there are works that reduced the combinatorial search space by considering the composition of base functions, \textit{e.g.} fast function extraction~\citep{mcconaghy2011ffx} and elite bases regression~\citep{DBLP:conf/icnc/ChenLJ17}.  
In terms of the families of expressions, research efforts have been devoted to searching for polynomials with single or two variables~\citep{DBLP:journals/gpem/UyHOML11}, time series equations~\citep{DBLP:conf/icml/BalcanDSV18}, and  equations in physics~\citep{journal/2020/aifrynman}. 

Multi-variable symbolic regression is challenging because the search space increases exponentially with respect to the number of independent variables. Existing works on multi-variable regression are mainly based on pre-trained encoder-decoder methods with massive training datasets (e.g., {millions of data points}~\citep{DBLP:conf/icml/BiggioBNLP21}), and even larger-scale generative models (e.g., approximately 100 million parameters~\citep{DBLP:conf/nips/KamiennydLC22}).
Our \method algorithm is a tailored algorithm to solve multi-variable symbolic regression problems. 

% \xyx{Merge this part in.}
% Our \cvmt is also well-connected to another line of work on Lagramge system~\citep{DBLP:conf/icml/TodorovskiD97,DBLP:conf/dis/GanzertGSK10,DBLP:journals/kbs/BrenceTD21}. Both of them needs to search for a sequence of production rules in context-free grammars of symbolic expression.

%Recent works exploit diverse-structured deep neural networks to discover interpretable governing equations from the data, which greatly reduces the redundant workload for a scientist in different domains. 
% This website: \url{http://cogsys.org/symposium/discovery-2022/}
\medskip
\textbf{Active Learning.}
Active learning considers querying data points actively to maximize the learning performance~\citep{DBLP:journals/ftml/Hanneke14,golovin2010near}. Recently, Haut et al. has applied active learning to better query data to accelerate the discovery process~\citep{DBLP:conf/gecco/HautBP22,DBLP:conf/gecco/HautPB23}.
Our approach is related to active learning because control variable experiments can be viewed as a way to actively collect data. 
However, in addition to active data collection, our \method builds
simple to complex models, which have not been explored in active learning. 

\medskip
\textbf{Meta-reasoning -- Thinking Fast and Slow.}
The co-existence of fast and slow cognition systems marks an interesting side of human intelligence~\citep{kahneman2011thinking,DBLP:conf/nips/AnthonyTB17,DBLP:conf/aaai/BoochFHKLLLMMRS21}. 
Our \method is motivated by this dual cognition process. % of n. 
In essence, we argue that instead of entirely relying on the brute-force way of learning with big data and heavy computation (fast thinking), careful meta-reasoning on the strategies to determine ground-truth equations (slow thinking), e.g. incrementally expanding from reduced-form equations to the full equation, may result in better outcomes.
% AAAI think fast and slow symposium~\url{https://sites.google.com/view/aaai-fss22?pli=1}

\medskip
\textbf{Causality.}
Control variable experiments are closely related to the idea of intervention, which is commonly used to discover causal relationships \citep{simon1954spurious,langley2019scientific,glymour2014discovering,jaber2022causal,pearl2009causality}. 
However, we mainly use control variable experiments to accelerate symbolic regression, which still identifies correlations instead of the causal relationships.

%% file: tex/8.1.exp-set.tex
In this section, we demonstrate that \method is superior to multiple baselines in the following way:
\begin{itemize}
\item \method finds the expressions with the smallest median Normalized Mean-Square Errors (NMSE) 
among all 7 competing approaches on noiseless datasets (in Table~\ref{tab:Trigonometric-nmse-noiseless} and Table~\ref{tab:feynman-livermore2-nmse-noiseless}) and 
20 noisy benchmark datasets (in Table~\ref{tab:Trigonometric-nmse-noisy}). In particular, \cvgp attains the best empirical results on datasets with a large number of variables while \cvmt is the best on datasets with a median number of variables.
\item On simpler equations, we show that our \method takes less training time and memory, but has a higher rate of recovering the ground-truth expressions compared to baselines following the horizontal paths (in Table~\ref{tab:recovery}). 
\item We show that our \method method is consistently better than the baselines under different evaluation metrics (in Fig.~\ref{fig:evalucate-metric}),  different quantiles (25\%, 50\% and 75\%) of the NMSE metric (in Fig.~\ref{fig:Quartile-trig-nmse-noiseless-partial}), and with different amounts of Gaussian noise added to the data (in Fig.~\ref{fig:noise-level-metric}). 
\end{itemize}

\subsection{Experimental Settings} \label{sec:exp-set}
This section briefly discusses the choice of datasets, baselines, evaluation criteria, and training/testing settings.

\subsubsection{Choice of Datasets} 
We mainly consider several popular and large-scale datasets involving multiple variables that are used in prior research on symbolic regression tasks. 
\begin{itemize}
\item The Trigonometric Datasets~\citep{DBLP:arxiv/ecml/ny23}. These are a series of synthesized datasets, composed of randomly generated expressions with multiple variables. 
In this series, a dataset is labeled by the ground-truth equation that generates it. 
The ground-truth equations are multi-variable polynomials characterized by their operands and a tuple $(l_1, l_2, l_3)$. $l_1$ is the number of independent variables, $l_2$ is the number of singular terms, and $l_3$ is the number of pairwise terms.
A singular term can be an independent variable, such as $x_1$, or a unary operator on a variable, such as $\sin(x_1)$. The pairwise terms  look like $c_1 x_3 x_4$ or $c_2 \sin(x_1) \texttt{inv}(x_5)$, etc.
Here $c_1$ and $ c_2$ are randomly generated constants. 
The tuples and operands listed in different tables and charts indicate how the ground-truth expressions are generated. 
There are a series of $21$ dataset configurations.
For each configuration, there are $10$ random expressions. %The experimental results over these 10 expressions are reported in charts and tables. 

\item The Feynman Datasets~\citep{DBLP:conf/nips/UdrescuTFNWT20}. This dataset includes 120 equations from Richard Feynman's famous physics textbook. Because the difficulty of discovery is mainly determined by the number of variables $m$ in each equation, we partition the whole dataset into 6 groups according to the number of variables $m$.

\item The Livermore2 Datasets~\citep{DBLP:conf/iclr/PetersenLMSKK21} is a randomly generated multiple variable dataset including equations with $2\le m\le 7$ variables. For every variable setting, there are 25 expressions. For some of the equations, they are numerically evaluated with values like infinity and Not-a-Number. Even fitting the constant values (assuming access to ground-truth equation forms) using a gradient-based optimizer causes numeric issues. Because the purpose of this paper is not to optimize numeric solvers, we modified some expressions from the original paper~\citep{DBLP:conf/iclr/PetersenLMSKK21} to avoid numeric stability issues.
\end{itemize}
For each dataset, they are all partitioned into groups according to the number of input variables.
For the Trigonometric datasets, every group of expressions has the same set of operators. However, for the Feynman and Livermore2 datasets, every group of expressions may involve a different set of operators.
The exact form of the expression in each dataset is provided in Appendix~\ref{apx:dataset-config}.

\medskip
\textbf{Remarks on Public Available Datasets.} 
Most public datasets are black-box \citep{la2021contemporary}, containing randomly generated input and output pairs of an unknown symbolic equation. 
The purpose of our paper is to demonstrate the performance of vertical symbolic regression -- which requires us to access data in which the values of certain variables are controlled.
Since our \method requires knowing the ground-truth expressions, the Penn Machine Learning Benchmarks (PMLB) dataset~\citep{DBLP:journals/bioinformatics/RomanoLCGGCRHFM22} are not included because they do not have known ground-truth expressions.
Also, we intentionally test on benchmark sets involving many variables to highlight our approach. 
Because we consider the expressions of multiple variables, datasets that mainly consist of expressions with $\le 2$ variables are not considered, including Keijzer~\citep{DBLP:conf/eurogp/Keijzer03}, Korns~\citep{DBLP:conf/gptp/Korns14} and Constant~\citep{DBLP:conf/iclr/PetersenLMSKK21} datasets.

\medskip
\textbf{Noiseless and Noisy Settings.} The noiseless setting is used to determine if the algorithm is able to find the correct expression under the most ideal setting. The noisy setting is a simulation of the real world where experimental outcomes are measured with rounding errors and human mistakes. These datasets are used to determine the robustness of the symbolic regression algorithms. For noiseless datasets, the output $y_i$ is exactly the evaluation of the ground-truth expression $\phi(\mathbf{x}_i)$. For noisy datasets, the output is further perturbed by the noise of zero means and a given standard deviation: $\phi(\mathbf{x}_i)+\varepsilon_i$, where the noise $\varepsilon_i\sim\mathcal{N}(0,0.1)$ is drawn from the Gaussian distribution with a mean of $0$ and a standard deviation of $0.1$.

\subsubsection{Choice of Baselines}
We evaluate these symbolic regression methods:
1) Genetic programming-based approaches, including GP and Eureqa. 
2) A Monte Carlo Tree Search-based approach, \textit{i.e.}, MCTS.
3) Reinforcement-learning-based approaches, including DSR, PQT and GPMeld.  %The included methods all consider using a recurrent neural network to predict a pre-order traversal of the expression tree. 
The detailed description of each method is as follows:
\begin{itemize}%[align=left, leftmargin=0pt, labelwidth=0pt, itemindent=!]
    \item Genetic Programming (GP)~\citep{DEAP_JMLR2012} maintains a population of candidate symbolic expressions, in which this population \textit{evolves} between generations. In each generation, candidate expressions undergo \textit{mutation} with probability $P_{mu}$ and \textit{crossover} with probability $P_{ma}$. Then in the \textit{selection} step, expressions with the highest fitness scores (measured by the difference between the ground truth and candidate expression evaluation) are selected as the candidates for the next generation, together with a few randomly chosen expressions, to maintain diversity. After several generations, expressions with high fitness scores, \textit{i.e.}, those expressions that fit the data well survive in the pool of candidate solutions. The best expressions in all generations are recorded as {hall-of-fame} solutions.  
    \item {Eureqa~\citep{DBLP:journals/gpem/Dubcakova11} is the current best commercial software based on evolutionary search algorithms. Eureqa works by uploading the dataset $\mathcal{D}$ and the set of operators as a configuration file to its commercial server. Computation is performed on its commercial server and only the discovered expression will be returned after several hours.}
    \item MCTS~\citep{DBLP:conf/iclr/Sun0W023} that uses Monte Carlo Tree Search to find the best expressions, that is defined with context-free grammar in section~\ref{sec:cv-mcts}. For historical reasons, This method got completely different names in a line of works~\citep{DBLP:conf/icml/TodorovskiD97,DBLP:conf/dis/GanzertGSK10,DBLP:journals/kbs/BrenceTD21,DBLP:conf/iclr/Sun0W023,DBLP:conf/icml/KamiennyLLV23}, despite being implemented in ways similar to each other. Our implementation, as a representation of all these methods, is closest to \cite{DBLP:conf/iclr/Sun0W023}. %MCTS to avoid confusion. 
    \item Deep Symbolic Regression (DSR)~\citep{DBLP:conf/iclr/PetersenLMSKK21} uses a combination of recurrent neural network (RNN)  and reinforcement learning for symbolic regression. The RNN generates possible candidate expressions, and is trained with a risk-seeking policy gradient objective to generate better expressions.
    \item Priority queue training (PQT)~\citep{DBLP:journals/corr/abs-1801-03526} also uses the RNN similar to DSR for generating candidate expressions. However, the RNN is trained with a supervised learning objective over a data batch sampled from a maximum reward priority queue, focusing on optimizing the best-predicted expression.
    \item Vanilla Policy Gradient (VPG)~\citep{DBLP:journals/ml/Williams92} is similar to DSR method for the RNN part. The difference is that VPG uses the classic REINFORCE method for computing the policy gradient objective.
    
    \item Neural-Guided Genetic Programming Population Seeding (GPMeld)~\citep{DBLP:conf/nips/MundhenkLGSFP21} uses the RNN to generate candidate expressions, and these candidate expressions are improved by a genetic programming (GP) algorithm. 
\end{itemize}

\subsection{Evaluation Criteria}
In terms of the evaluation criteria,  we consider the following:
\begin{itemize}
    \item Goodness-of-fit metric. The median (50\%) of the NMSE values to fit all the expressions in a dataset is reported. 
We choose to report median values instead of means due to outliers (see box plots in Fig.~\ref{fig:Quartile-trig-nmse-noiseless-partial}). This is a common practice for combinatorial optimization problems. We further report the performance on other metrics (see equation~\ref{eq:loss-function}) in case studies Fig.~\ref{fig:evalucate-metric}. We set a general total time to be 48 hours to ensure all programs are finished and well-trained for this metric.
\item The total running time of each learning algorithm, which is the duration taken for each program to uncover a promising expression. It is worth noting that this calculation incorporates the time spent on data oracle queries and optimizes the open constants in each expression.
\item Memory consumption of the learning algorithms, measuring the peak memory consumption on the set of optimal expressions maintained by the learning algorithm, gradient-based optimizer for open constants, training data size, and deep network size (if used). The memory of the whole program initialization and the data Oracle initialization are excluded for comparison.
\end{itemize}

For the goodness-of-fit metric, given a testing dataset $D_{\text{test}}=\{(\mathbf{x}_{i},y_i)\}_{i=1}^n$ generated from the ground-truth expression, we measure the goodness-of-fit of a predicted expression ${\phi}$, by evaluating the mean-squared-error (MSE), normalized-mean-squared-error (NMSE),  root Mean-squared error (RMSE), and normalized root Mean-squared error (NRMSE):
\begin{equation}\label{eq:loss-function}
\begin{aligned}
\text{MSE}(\phi)&=\frac{1}{n}\sum_{i=1}^n(y_{i}-{\phi}(\mathbf{x}_{i}))^2,  \\
\text{NMSE}(\phi)&=\frac{\frac{1}{n}\sum_{i=1}^n(y_{i}-{\phi}(\mathbf{x}_{i}))^2}{\sigma_y^2} \\
\text{RMSE}(\phi)&=\sqrt{\frac{1}{n}\sum_{i=1}^n(y_{i}-{\phi}(\mathbf{x}_{i}))^2},\\ 
\text{NRMSE}(\phi)&=\frac{1}{\sigma_y} \sqrt{\frac{1}{n}\sum_{i=1}^n(y_{i}-{\phi}(\mathbf{x}_{i}))^2}
\end{aligned}
\end{equation}
where the empirical variance $\sigma_y$ is computed as $\sqrt{\frac{1}{n}\sum_{i=1}^n \left(y_i-\frac{1}{n}\sum_{i=1}^n y_i\right)^2}$. Additionally,  the Inverse normalized Mean-squared error (InvNMSE), and Inverse normalized root Mean-squared error (InvNRMSE) are defined as follows:
\begin{align*}
\text{InvNMSE}(\phi)&=\frac{1}{1+\text{NMSE}(\phi)},\\
\text{InvNRMSE}(\phi)&=\frac{1}{1+\text{NRMSE}(\phi)}.
\end{align*}
We use the NMSE as the main criterion for comparison in the experiments and present the results on the remaining metrics in the case studies. 
The main reason is that the NMSE is less impacted by the output range. The output ranges of expression are dramatically different from each other, making it difficult to present results in a uniform manner if we use other metrics.
Prior work~\citep{DBLP:conf/iclr/PetersenLMSKK21} further proposed coefficient of determination ($R^2$)-based Accuracy over a group of expressions in the dataset, which is defined as: given a threshold $\tau$ (such as $\tau=0.999$), for a dataset containing fitting tasks of $N$ expressions, the algorithm finds a group of best expressions $[\phi_1,\ldots, \phi_N]$ correspondingly. The $R^2$-based accuracy is computed as follows:
\begin{align*}
\text{Accuracy}(R^2\ge \tau)&=\frac{1}{N}\sum_{i=1}^{N}\mathbf{1}(R^2(\phi_i)\ge \tau),
\end{align*}
where $R^2(\phi_i)=1-\frac{\frac{1}{n}\sum_{i=1}^n(y_{i}-{\phi}(\mathbf{x}_{i}))^2}{\sigma_y^2}$ and $\mathbf{1}(\cdot)$ is an indicator function that outputs $1$ when the $R^2(\phi_i)$ exceeds the threshold $\tau$.
Note that the coefficient of determination ($R^2$) metric~\citep{nagelkerke1991note} is equal to $1-\text{NMSE}(\phi)$.

\subsection{Training and Testing Settings} \label{sec:experiment-test-setting}
We leave detailed descriptions of the configurations of our methods and baselines in Appendix~\ref{apx:impelemnt} and only mention a few implementation notes here. 
We implemented the GP, \cvgp, \mcts, and \cvmt. 
They use a data oracle, which returns (noisy) observations of the ground-truth equation when queried with inputs. 
We cannot implement the same oracle for other baselines because of code complexity and/or no available code. 
For PQT, VPG, DSR, and GPMeld, we generate a large fixed-size dataset before training. During training, the method samples a small batch of data $\{(\mathbf{x}_i, y_i)\}_{i=1}^{m}$ for a step of mini-batch gradient descent.   See ``Training set size'' and ``Batch size'' in Appendix Table~\ref{tab:baseline-hyper-config} for empirical configurations.
To ensure fairness, the sizes of the training datasets we use for these baselines are larger than the total number of data points accessed in the full execution of those algorithms. 
In other words, their access to data would have no difference if the same oracle for GP, \cvgp, \mcts and \cvmt has been implemented for them because it does not affect the executions whether the data is generated ahead of the execution or on the fly.

To ensure the fairness of the testing, every learning algorithm outputs the most probable symbolic expression. We apply the same testing set ${D}_{\text{test}}$ to compute the goodness-of-fit measure in Equation~\eqref{eq:loss-function}. The reported NMSE scores in all the charts and tables are based on separately generated data that have never been used in training.

%% file: tex/8.2.exp.tex
\input{tex/8.2.noiseless.trigonometric}

\input{tex/8.2.noiseless.feynman_livermore2}

\subsection{Goodness-of-fit Comparisons}
\textbf{Results under Noiseless Settings.}
 Under the noiseless setting, we evaluate the median NMSE metric of all the algorithms over the Trigonometric dataset in Table~\ref{tab:Trigonometric-nmse-noiseless}, the Feynman dataset in Table~\ref{tab:feynman-livermore2-nmse-noiseless}(a) and the Livermore2 dataset in Table~\ref{tab:feynman-livermore2-nmse-noiseless}(b).

In Table~\ref{tab:Trigonometric-nmse-noiseless}, we find that \cvgp is better than GP and \cvmt is better than \mcts in terms of the median NMSE metric, showing that the proposed \method can sufficiently improve the current baselines based on horizontal discovery. Also, \cvgp and \cvmt can find better expressions than deep reinforcement learning-based baselines (i.e., DSR, PQT, VPG, and GPMeld). 
In Table~\ref{tab:Trigonometric-nmse-noiseless}(a,b), the \cvmt outperforms the \cvgp because the number of allowed mathematical operators is median. As a result,  the ground-truth expression only requires relatively shallow tree searches from the best expressions found in the previous vertical discovery step to reach a good candidate equation. For expressions with more than 6 variables and a large set of operators (in Table~\ref{tab:Trigonometric-nmse-noiseless}(c)), \cvmt needs much deeper expansions in the search tree to find a good candidate expression. This results in inferior performance compared with \cvgp.
Overall, our \method attains the smallest median (50\%) NMSE values among all the baselines mentioned in Section~\ref{sec:exp-set}, when evaluated on noiseless datasets (Table~\ref{tab:Trigonometric-nmse-noiseless}). 
This shows that our proposed methods based on vertical discovery can handle symbolic regression problems with many independent variables better than the current state-of-the-art algorithms in this area.

Table~\ref{tab:feynman-livermore2-nmse-noiseless}(a) collects the NMSE values for the Feynman dataset, the expressions of which are from real-world physics. We can find that \cvmt attains results better than the rest baselines.

The results on the Livermore2 dataset are presented in Table~\ref{tab:feynman-livermore2-nmse-noiseless}(b).
%The equations in the Livermore2 dataset contain a large size of operators compared to the Trigonometric dataset.
We can find that the results between RL-based methods and the GP-based methods are mixed. The reason is that these datasets offer a larger size of mathematical operators than the other two datasets, GP-based methods need many more random mutations and a larger set size of GP pool to find good candidate expressions. 
For the MCTS  and our \cvmt method, the corresponding search tree has many more children due to the larger size of mathematical operators. %, which enlarges the difficulty of finding the correct path from root to leaf within the time limit.
Nevertheless, we would like to point out that vertical discovery approaches (e.g., \cvgp and \cvmt) still outperform their horizontal counterparts (e.g., GP and MCTS).

\medskip
\textbf{Noisy Settings.} %We further compare all methods in \textit{noisy} setting on the Trigonometric datasets, which is shown in Table~\ref{tab:Trigonometric-nmse-noisy}.
We also conduct experiments on the Trigonometric dataset for noisy settings, to benchmark the robustness of the learning algorithm. The result is summarized in Table~\ref{tab:Trigonometric-nmse-noisy}. The Gaussian noise has a zero mean and a standard deviation of $0.1$ is added. In all except for one dataset, our approaches \cvgp and \cvmt attain the smallest NMSEs compared to all baselines. %\xyx{highlight the best results.

\input{tex/8.2.noisy.trigonometric}

%% file: tex/8.2.noiseless.trigonometric.tex
\begin{table}[!t]
    \centering
    \caption{On \textit{noiseless} Trigonometric datasets, median (50\%-quartile) of NMSE values of the best-predicted expressions found by all the algorithms. The 3-tuples at the top $(\cdot,\cdot,\cdot)$ indicate the number of free variables, singular terms, and cross terms in the ground-truth expressions generating the dataset. $O_p$ stands for the set of operators. Our \cvgp and \cvmt are the best (attains the smallest NMSE values) among all approaches.}
    \label{tab:Trigonometric-nmse-noiseless}
    \begin{tabular}{l|rrrrrrr}
     \multicolumn{8}{c}{\textbf{(a)} Trigonometric datasets containing $O_p=\{\texttt{inv},+,-,\times\}$} \\
\toprule
& (2,1,1) & (3,2,2) & (4,4,6) & (5,5,5) & (5,5,8) & (6,6,8) & (6,6,10) \\   \midrule
\cvgp(ours) & $<$\textbf{1E-6} & 1E-3 & {0.008} & {0.011} & {0.007} & {0.044} & {0.012} \\
\gp & 2E-3 & 0.015 & 0.012 & 0.025 & 0.010 & 0.058 & 0.381 \\
 Eureqa & $<$\textbf{1E-6} & $<$\textbf{1E-6} & 1.191 & 0.996 & 1.002 & 1.005 & 1.764 \\
 \hline
\cvmt(ours)& $<$\textbf{1E-6} & $<$\textbf{1E-6}  & $<$\textbf{1E-6} & \textbf{6.8E-6} & \textbf{9.3E-5} & \textbf{9.2E-5} & \textbf{3.1E-5}\\
\mcts & $<$\textbf{1E-6}  &  0.059   &  0.495 & 0.466 & 0.667   &  0.661 & 0.590\\
\hline
DSR & $<$\textbf{1E-6} & 1.004 & 1.006 & 1.048 & 1.403 & 1.963 & 1.021 \\
PQT & $<$\textbf{1E-6} & 0.874 & 1.006 & 1.048 & 1.530 & 4.212 & 1.006 \\
VPG & $<$\textbf{1E-6}& 0.978 & 1.221 & 1.401 & 4.133 & 4.425 & 1.003 \\
GPMeld & $<$\textbf{1E-6} & 1.062 & 1.127 & 1.008 & 1.386 & 15.58 & 1.022 \\
\midrule
    \multicolumn{8}{c}{\textbf{(b)} Trigonometric datasets containing  $O_p=\{\sin, \cos,+,-,\times\}$} \\
    \toprule
& (2,1,1) & (3,2,2) & (4,4,6) & (5,5,5) & (5,5,8) & (6,6,8) & (6,6,10) \\   \midrule
\cvgp (ours)& 0.005 & 0.028 & 0.086 & 0.014 & 0.066 & {0.066} & {0.104} \\
GP & 7{E-}4 & 0.023 & 0.044 & 0.063 & 0.102 & 0.127 & 0.159 \\
Eureqa & $<$\textbf{1E-6} & $<$\textbf{1E-6} & 0.024 & 0.158 & 0.284 & 0.433 & 0.910 \\
\hline
\cvmt(ours)& $<$\textbf{1E-6} & $<$\textbf{1E-6}  & \textbf{0.006} & \textbf{0.009} & \textbf{0.011} & \textbf{0.014} & \textbf{0.076}\\
\mcts & 0.006 & 0.033 & 0.144 & 0.147 & 0.307 & 0.391 & 0.472 \\
\hline
DSR & $<$\textbf{1E-6} & 0.008 & 2.815 & 2.558 & 2.535 & 0.936 & 6.121 \\
PQT & 0.020 & 0.161 & 2.381 & 2.168 & 2.482 & 0.983 & 5.750 \\
VPG & 0.030 & 0.277 & 2.990 & 1.903 & 2.440 & 0.900 & 3.857 \\
GPMeld & $<$1{E-}6 & 0.112 & 1.670 & 1.501 & 2.422 & 0.964 & 7.393 \\
\midrule
    
 \multicolumn{8}{c}{\textbf{(c)} Trigonometric datasets containing $O_p=\{\sin, \cos,\texttt{inv},+,-,\times\}$} \\
\toprule 
& (2,1,1) & (3,2,2) & (4,4,6) & (5,5,5) & (5,5,8) & (6,6,8) & (6,6,10) \\   \midrule
\cvgp(ours) & $<$\textbf{1E-6} & 0.039 & 0.015 & 0.038 & \textbf{0.050} & \textbf{0.029} & \textbf{0.018} \\
GP & $<$\textbf{1{E-}6} & 0.043 & 0.042 & 0.197 & 0.111 & 0.091 & 0.087 \\
Eureqa & $<$\textbf{1{E-}6} & $<$\textbf{1{E-}6} & 0.259 & 0.901 & 1.006 & 1.002 & 1.001 \\
\hline
\cvmt(ours)& $<$\textbf{1E{-}6} & $<$\textbf{1E{-}6} & \textbf{6E\mbox{-}3} & \textbf{0.007} & {0.069} &  0.226 &  0.219 \\
\mcts & 0.010 & 0.119 & 0.330 & 0.482 & 0.453 & 0.476 & 0.484 \\
\hline
DSR & 0.439 & 0.233 & 1.040 & 3.892 & 0.782 & 1.605 & 2.083 \\
PQT & 0.485 & 0.855 & 1.039 & 4.311 & 1.217 & 1.718 & 1.797 \\
VPG & 0.008 & 0.227 & 1.049 & 5.542 & 0.572 & 4.691 & 1.888 \\
GPMeld & $<$\textbf{1E-6}  & 0.984 & 1.886 & 9.553 & 1.142 & 1.398 & 2.590 \\
\bottomrule
    \end{tabular}
\end{table}

%% file: tex/8.2.noiseless.feynman_livermore2.tex
\begin{table}[!t]
    \centering
    \caption{On \textit{noiseless} Feynman and Livermore2 datasets, median (50\% quantile) NMSE values of the best-predicted expressions found by all the algorithms. The full dataset is partitioned by the number of variables ($m$) contained in the ground-truth expression.  Our approaches \cvgp and \cvmt attain the smallest NMSEs compared to all baselines. }
    \label{tab:feynman-livermore2-nmse-noiseless}
    \begin{tabular}{l|rrrrcr}
    \toprule
    & \multicolumn{5}{c}{\textbf{(a) Feynman datasets.}} \\
        &  $m=2$  & $m=3$ & $m=4$ & $m=5$ & $6\le m \le 8$     \\ \midrule
\cvgp (ours)&   $<$ 1E-6 & $<$ 1E-6 & $<$ 1E-6  & $0.998$ &  {$1.006$} \\
GP  & $<$ 1E-6 &  $<$ 1E-6  &  0.936 & $1.081$  & {1.005}	 \\
   Eureqa & $<$ 1E-6 & $<$ 1E-6 & 0.026 & 0.434 & {0.80}\\
\hline
\cvmt  (ours)&  \textbf{$<$ 1E-6} & \textbf{$<$ 1E-6} & \textbf{$<$ 1E-6} & \textbf{0.065} & \textbf{0.144}\\
\mcts & $<$ 1E-6  & 8.1E-3& $1.003$ & 0.181 & {1.023}\\
\hline
DSR & 2.28\mbox{E-}4 &   0.222   &   0.216   &   0.976   &  0.908  \\
PQT  &   3.20\mbox{E-}4   &   0.191   &   0.172   &   1.003   &  1.383   \\
VPG  &   2.74\mbox{E-}4   &   0.155   &   0.188   &   1.006   & 1.435  \\
GPMeld  &   3.71\mbox{E-}4   &   0.182   &   0.177   &   0.941   &  1.366   \\
\midrule
    \midrule
     & \multicolumn{5}{c}{\textbf{(b) Livermore2 datasets.}} \\
        & $m=2$ & $m=3$ & $m=4$ & $m=5$ & $m=6$  &  $m=7$  \\
         \midrule
        \cvgp  (ours)& \textbf{$<$1E\mbox{-}6} & 0.057 & 0.013  &0.275  & 0.117 & 0.085\\
        GP &  \textbf{$<$1E\mbox{-}6 } & 0.068 &0.059 &0.331   &0.256 & 0.238 \\
        Eureqa & 0.991 & 0.051 &0.508   &0.083   &\textbf{0.026}  &  0.558 \\
        \hline
        \cvmt  (ours)&  6.9E\mbox{-}6 & 0.020 & \textbf{0.012} & 0.071 & 0.191 & 0.104  \\
        \mcts & 9.5E-3  & 0.058 &0.054 &0.181 &0.229 &0.103 \\
        \hline
        DSR &1.3E-5 & 0.012  & 0.030   & 0.050   & 0.230 & \textbf{0.073} \\
        PQT & 3.9E-6 & 0.017  & 0.042 & 0.074 & 0.170  & 0.074\\
        VPG & 5.9E-6  & 0.031 & 0.037   & 0.093   & 0.206 & 0.078\\
        GPMeld  & \textbf{$<$ 1E-6} & \textbf{0.002}  &0.029   &\textbf{0.049}  &0.144 & 0.104\\
        \bottomrule
    \end{tabular}%
\end{table}

%% file: tex/8.2.noisy.trigonometric.tex
\begin{table}[!t]
    \centering
     \caption{On \textit{noisy} Trigonometric datasets, median (50\% quantile) NMSE values of the best expressions found by all the algorithms. The Gaussian noise has zero mean and a standard deviation of $0.1$ is added. In all except for one dataset, our approaches \cvgp and \cvmt attain the smallest NMSEs compared to all baselines. }
    \label{tab:Trigonometric-nmse-noisy}
    \begin{tabular}{l|rrrrrrr}
     \multicolumn{8}{c}{\textbf{(a) Trigonometric datasets containing $O_p=\{\texttt{inv},+,-,\times\}$}} \\
    \toprule
        & $(2,1,1)$ & $(3,2,2)$ & $(4,4,6)$ & $(5, 5, 5)$ & $(5, 5, 8)$ & $(6, 6, 8)$ & $(6, 6, 10)$ \\ \midrule
\cvgp  (ours) &${0.011}$ & $0.020$  &${0.036}$  &$0.076$  &$\textbf{0.061}$  &$\textbf{0.098}$  &$\textbf{0.055}$  \\
GP  &${0.02}$& $0.031$   &${0.088}$  &$0.126$  &$0.118$  &$0.144$ & $0.097$ \\
% Eureqa & \\
\hline
\cvmt  (ours)& $\textbf{0.001}$ & $\textbf{1.7E\mbox{-}4}$ & $\textbf{0.011}$ & $\textbf{0.071}$  & $0.121$ & $0.153$ & $0.117$\\
\texttt{MCTS} &  $\textbf{0.001}$ & $2.6E\mbox{-}4$ & $0.191$ & $0.240 $ & $0.350$ & $0.520$ & $0.449$\\
\hline
DSR  &$0.03$  & $0.58$  &$1.163$  &$1.028$  &$1.004$  &$1.006$  &$1.003$  \\
PQT  &$0.03$  & $0.43$  &$1.016$  &$1.983$  &$1.005$  &$1.006$  &$1.005$  \\
VPG  &$0.04$  & $0.62$  &$1.09$  &$1.08$  &$1.00$  &$1.01$  &$1.00$  \\
GPMeld  &$0.39$& $0.55$  &$1.058$  &$1.479$  &$1.108$  &$1.035$  &$1.021$  \\
    \midrule
     \multicolumn{8}{c}{\textbf{(b) Trigonometric datasets containing $O_p=\{\sin, \cos,+,-,\times\}$}} \\
 \toprule   
   &  $(2,1,1)$ & $(3,2,2)$ & $(4,4,6)$ & $(5,5,5)$ & $(5,5,8)$ & $(6,6,8)$ & $(6,6,10)$ \\\midrule
\cvgp  (ours) &${0.05}$  &${0.10}$  &${0.08}$  &$\textbf{0.07}$  &$\textbf{0.11}$  &${0.17}$  &$\textbf{0.16}$  \\
GP  &$0.10$  &$0.11$  &$0.12$  &$0.09$  &$0.12$  &$0.19$  &$0.31$  \\
% Eureqa & \\
\hline
\cvmt  (ours)& $\textbf{0.005}$ & $\textbf{0.012}$  & $\textbf{0.051}$ & $0.092$ & $0.148$ & $\textbf{0.121}$ & $0.340$\\
\texttt{MCTS} & $0.015$  & $0.007$ & $0.138$ & $0.150$ & $0.252$ & $0.244$ & $0.494$\\
\hline
DSR  &$0.07$  &$0.35$  &$7.06$  &$32.5$  &$195.2$  &$1.75$  &$11.68$  \\
PQT  &$0.07$  &$0.35$  &$5.09$  &$36.80$  &$449.8$  &$4.89$  &$5.67$  \\
VPG  &$0.09$  &$0.44$  &$2.46$  &$14.44$  &$206.1$  &$2.40$  &$7.40$  \\
GPMeld  &$0.07$  &$0.102$  &$2.225$  &$28.440$  &$363.79$  &$1.478$  &$11.513$  \\
   \midrule
  \multicolumn{8}{c}{\textbf{(c) Trigonometric datasets containing $O_p=\{\sin, \cos,\texttt{inv},+,-,\times\}$}} \\
\toprule
&   $(2,1,1)$ & $(3,2,2)$ & $(4,4,6)$ & $(5,5,5)$ & $(5,5,8)$ & $(6,6,8)$ & $(6,6,10)$ \\ \midrule
\cvgp  (ours)&$0.24$  &$0.05$  &${0.14}$  &$\textbf{0.16}$  &$\textbf{0.12}$  &${0.21}$  &${0.14}$  \\
GP &$0.10$  &${0.02}$  &$0.24$  &$0.20$  &$0.17$  &${0.21}$  &$\textbf{0.07}$  \\
% Eureqa & \\
\hline
\cvmt  (ours)& $\textbf{2E-3}$ & $\textbf{0.011}$ & $\textbf{0.128}$ & $0.180$ & $0.139$& $\textbf{0.162}$ & $0.251$\\
\texttt{MCTS} &  $0.003$ & $0.025$ & $0.263$ & $0.434$ & $0.393$ & $0.410$ & $0.528$\\
\hline
DSR &$0.44$  &$0.66$  &$1.031$  &$1.098$  &$1.009$  &$1.003$  &${1.654}$  \\
PQT &$0.76$  &$1.002$  &$1.297$  &$1.018$  &$1.017$  &$1.047$  &$1.027$  \\
VPG &$0.21$  &$0.969$  &$1.051$  &$1.012$  &$1.007$  &$1.059$  &$1.009$  \\
GPMeld & $\textbf{2\text{E-}3}$  &$0.413$  &$1.093$  &$1.036$  &$1.070$  &$1.029$  &$1.445$  \\
 \bottomrule
    \end{tabular}
\end{table}

%% file: tex/8.3.ablation.tex
\subsection{Case Studies}
This section studies the running time and memory consumption of each algorithm under suitable hyper-parameter configurations. We also study several other important topics, including 1) the choice of optimizers, 2) the distribution of full quartiles of every method, 2) the impact of the noise rate on the learning algorithm, and 3) the discovery rate of the ground-truth expressions.

\begin{table}[!t]
    \centering
      \caption{Recovery quality comparison. Our \method has a higher rate of recovering the ground-truth expressions compared to horizontal discovery on 3 simple datasets. Results are collected with a time limit of 48 hours. Our proposed \cvgp and \cvmt recover the highest percentage of ground-truth equations, run the fastest, and require the least amount of memory.}
    \label{tab:recovery}
    \begin{tabular}{l|cc|cc|cc}
    \multicolumn{7}{c}{\textbf{(a) Trigonometric datasets containing $O_p=\{\texttt{inv},+,-,\times\}$}.} \\
    \toprule
    & \multicolumn{2}{c|}{{\textbf{Accuracy}}} & \multicolumn{2}{|c|}{\multirow{2}{*}{\textbf{Total Time  (Mins)} $\downarrow$}} & \multicolumn{2}{|c}{\multirow{2}{*}{\textbf{Peak Memory (MB)} $\downarrow$}} \\
    & \multicolumn{2}{c|}{{\textbf{$(R^2\ge 0.999) \uparrow$}}} &  &\\
    & $(2,1,1)$ & $(3, 2, 2) $ & $(2,1,1)$ & $(3, 2, 2) $  & $(2,1,1)$ & $(3, 2, 2) $  \\   \midrule
    \cvgp(ours) & ${60\%}$  & $70\%$ & 8 & 10 &36 & \textbf{40}\\
    GP & $40\%$ & $40\%$ & 2 & 21 & 42 & 49\\
    \hline
    \cvmt (ours) &  \textbf{100\%} & \textbf{70\%}  & \textbf{2}  & \textbf{5} & \textbf{21} & 47  \\
    \mcts & $100\%$ & $40\%$  & 5 & 38 &   50 &  182\\
    \midrule
    \multicolumn{7}{c}{\textbf{(b) Trigonometric datasets containing $O_p=\{\sin,\cos,+,-,\times\}$}} \\
    \toprule
     & $(2,1,1)$ & $(3, 2, 2) $ & $(2,1,1)$ & $(3, 2, 2) $  & $(2,1,1)$ & $(3, 2, 2) $  \\   \midrule
    \cvgp(ours) & ${60\%}$ & 50\% & 3 & 18 &  36& \textbf{37}\\
    GP & $50\%$ & $40\%$ & 3 & 25& 40 &  45 \\
    \hline
    \cvmt(ours) & \textbf{100\%} & \textbf{100\%}   & \textbf{2} & \textbf{8} & \textbf{25} & 61\\
    \mcts & $10\%$ &   $20\%$  & 23 & 249 &  141 &  191\\
    \midrule
    \multicolumn{7}{c}{\textbf{(c) Trigonometric datasets containing $O_p=\{\sin,\cos,\texttt{inv},+,-,\times\}$}} \\
    \toprule
     & $(2,1,1)$ & $(3, 2, 2) $ & $(2,1,1)$ & $(3, 2, 2) $  & $(2,1,1)$ & $(3, 2, 2) $  \\   \midrule
    \cvgp(ours) & ${60\%}$ & 20\% & 8 & 28  & 37 & \textbf{36}\\
    GP & $60\%$ & 30\% & 13 & 11 & 42 & 45 \\
    \hline
    \cvmt(ours) & \textbf{100\%} & \textbf{70\%}  & \textbf{3} & \textbf{17} & \textbf{26 }& {83}\\
    \mcts & $ 0\%$ & 0\% & 50 & 287 & 144 & 206\\
\bottomrule
    \end{tabular}
\end{table}

\medskip
\textbf{Recovery Quality Comparison.}
We consider the task of discovering exact expressions in less challenging data sets. For a group of datasets, we
claim that a symbolic regressor recovers
the ground-truth equation if the $R^2$ values of the fitted equation are beyond 0.999. 
We hand-checked the equations found. They are basically ground-truth equations; many times in equivalent representations, e.g., use $x + x$ in placement of $2x$, etc. 
In the first column of Table \ref{tab:recovery}, we count the percentage of equations in each symbolic regressor that exceeds this 0.999 $R^2$ threshold. In other words, this is the percentage of equations that each approach is able to recover exactly. 
The computation time and memory usage are listed in the second and third columns.
For all the methods, we set the time limit to be 48 hours, the optimizer to be BFGS, and the maximum iterations of the optimizer to be 500. 
Table~\ref{tab:recovery}, our \method greatly improves the recovery quality.

Our \cvgp has a higher chance to recover ground-truth expressions than GP. Also, it requires less memory and time. 
Through hand checking, the expressions found by GP are longer and more complex than \cvgp. Hence, they require larger memory and more time for the BFGS optimizer to search for the constant values.
This serves as good empirical evidence that vertical symbolic regression reduces 
the search space of candidate expressions.

Compared to \mcts, our \cvmt methods use less than two hours to discover more than 70\% ground-truth equations on the $(3,2,2)$ dataset, while \mcts takes 5 hours discovering 0\% of ground-truth equations. 
Due to the vertical discovery paths, our \cvmt maintains much simpler equations than \mcts, resulting in less memory usage. Empirically, we found that the search for constant values is the main bottleneck for all these approaches. Hence, simpler equations also translate to much less running time. %Also, \cvmt searches for those simpler expressions with much less open constants to optimize than \mcts, which saves a great portion of time in the optimizer step.

We found many algorithms in our comparisons tend to give equivalent, but more complex representations of one equation, e.g., using $x + x + x$ in place of $3 x$. 
%\xyx{change to here.}
We do not have a good tool to simplify 
these equations into a canonical form. 
In this case, we omit the metric of 
normalized tree edit distance. 
For the challenging benchmark that we consider (benchmark involving a lot of variables, e.g., shown in tables \ref{tab:Trigonometric-nmse-noiseless} and \ref{tab:Trigonometric-nmse-noisy}), the equations found by all approaches tend to have big edit distances, hence not very informative. 
We suspect this is because the equations were not converted to a canonical form. 

%Note that the normalized tree edit distance is omitted here, due to the current API cannot compare long and complex expression trees that are numerically evaluated the same. 
%For datasets involving more input variables, the accuracy is nearly zero, making it hard to show the relative difference in recovery quality. Thus they are also omitted.

\medskip
\textbf{Empirical Running Time Analysis.} 
We further show the running time analysis in Fig.~\ref{fig:quartile-time-partial}. With more variables in the expression, the computational time to process these input data quickly scale up. Our proposed \cvgp and \cvmt have comparable running times as other approaches.

\medskip
\textbf{Impact of Different Optimizers.} 
Here we study the impact of using different optimizers in the search for constant values given the form of an equation. 
Notice such optimizations are often non-convex. 
We consider the following optimizers:  Conjugate gradient (CG)~\citep{fletcher1964function} Nelder-Mead~\citep{DBLP:journals/coap/GaoH12}, BFGS~\citep{bfgs}, Basin Hopping~\citep{wales1997global}, SHGO~\citep{DBLP:journals/jgo/EndresSF18}, Dual Annealing~\citep{tsallis1996generalized} and Dividing Rectangle (Direct)~\citep{nicholas2014dividing}.
The list of local and global optimizers shown in Fig.~\ref{tab:optimizer} are from Scipy library\footnote{\url{https://docs.scipy.org/doc/scipy/reference/optimize.html}}.
Empirically, we observe sometimes for a structurally correct equation (equivalent in structure to the ground-truth equation), an optimizer may find the values of open constants with large fitting errors. 
This fact places this equation in low ranking in the whole population. This is tragic because such expressions will not be considered after several rounds of GP or MCTS operations. 

\begin{figure}[!t]
    \centering
     \includegraphics[width=0.41\linewidth]{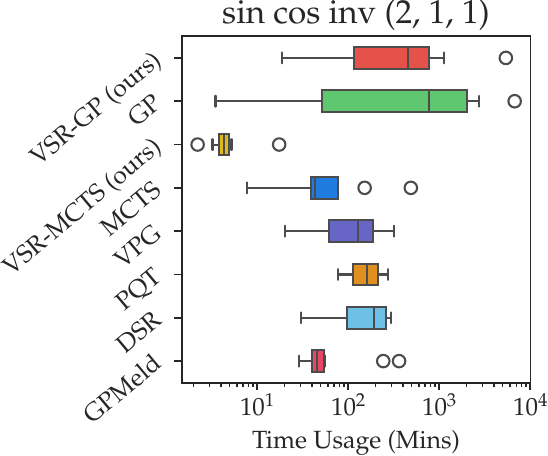}
     \includegraphics[width=0.27\linewidth]{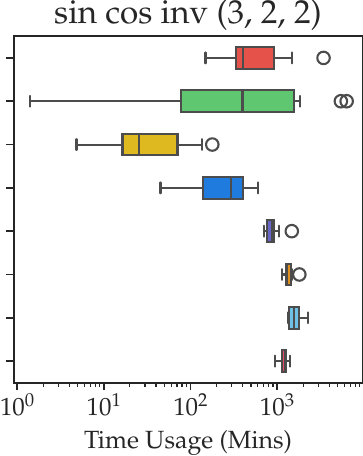}
     \includegraphics[width=0.27\linewidth]{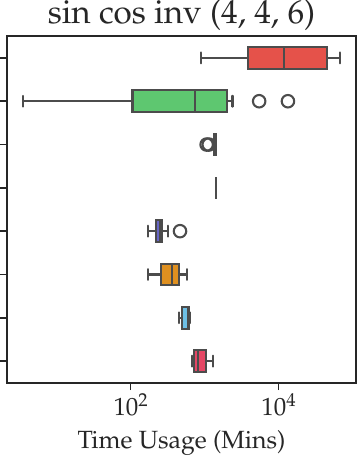}
    \caption{On \textit{noiseless} selected Trigonometric datasets, quartiles of the total running time of all the methods.}
    \label{fig:quartile-time-partial}
\end{figure}

We summarize the experimental result in Fig.~\ref{tab:optimizer}. In general, the list of global optimizers (SHGO, Direct, Basin-Hopping, and  Dual-Annealing) fits better for the open constants than the list of local optimizers but they take significantly more CPU time and memory. % for computations.

\begin{figure}
    \centering
    \includegraphics[width=0.49\linewidth]{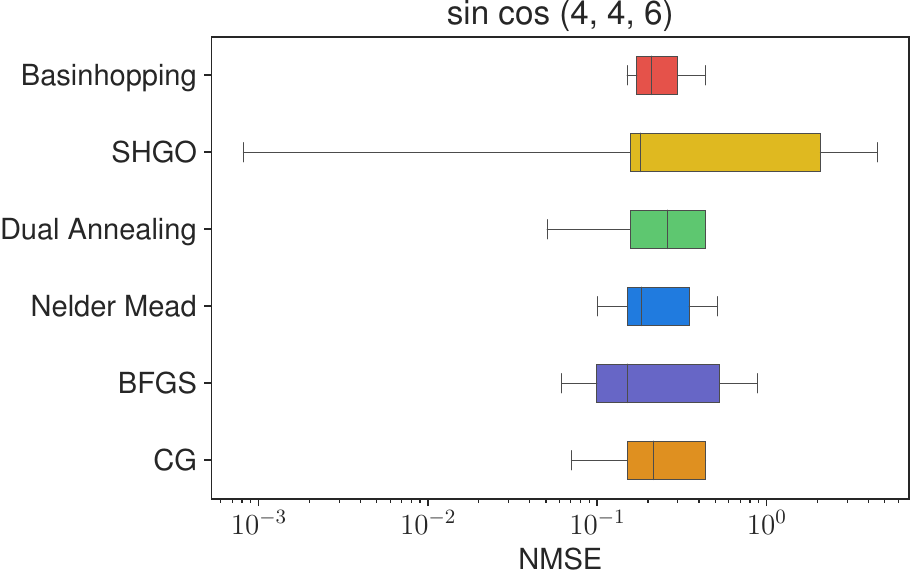}
    \hfill
    \includegraphics[width=0.49\linewidth]{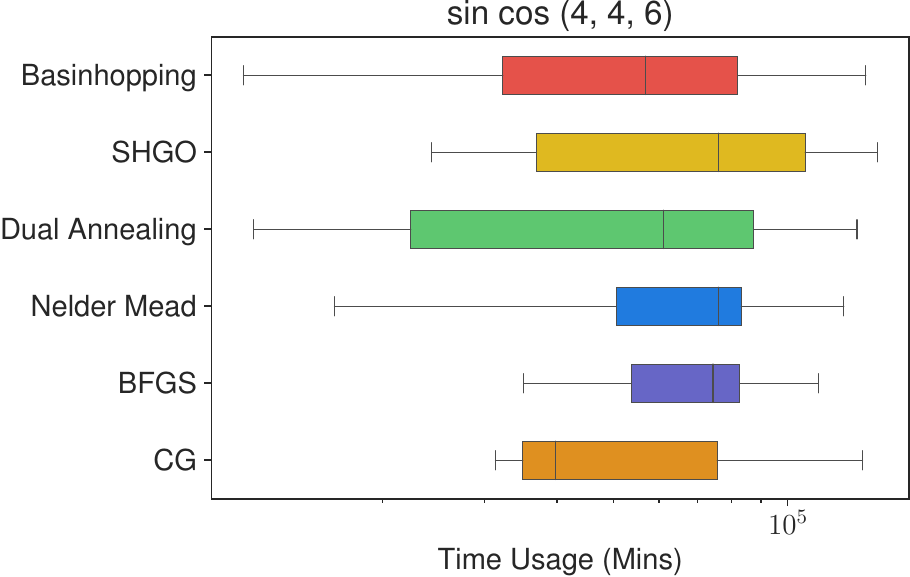}
    \caption{Impact of optimizers on finding the values of open constants for non-convex expressions. Over 10 randomly generated expressions involving 4 variables,  SHGO can find better solutions (in terms of NMSE metric) than local optimizers (including Nelder-Mead, BFGS, CG), while the time taken by SHGO is higher than local optimizers. }
    \label{tab:optimizer}
\end{figure}

\medskip
\textbf{Impact of Noise Levels.}
 In real scientific experiments, the datasets often contain noises.
We add Gaussian noise $\mathcal{N}(0, \sigma^2)$ to the output $y$ in the dataset and control the noise rate by varying the values of $\sigma$ in $\{0.02,0.04,0.08,0.1, 0.12, 0.14\}$.
Fig.~\ref{fig:noise-level-metric} shows the box plots in NMSE values for the expressions
found by \cvgp and GP over benchmark datasets with different noise levels. Our
\cvgp is consistently the best regardless of the evaluation metrics and noise levels.

\begin{figure}[!t]
    \centering
    \includegraphics[width=0.4\linewidth]{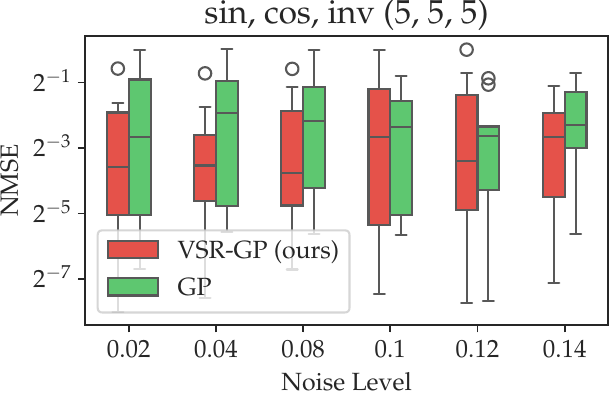}
    \includegraphics[width=0.4\linewidth]{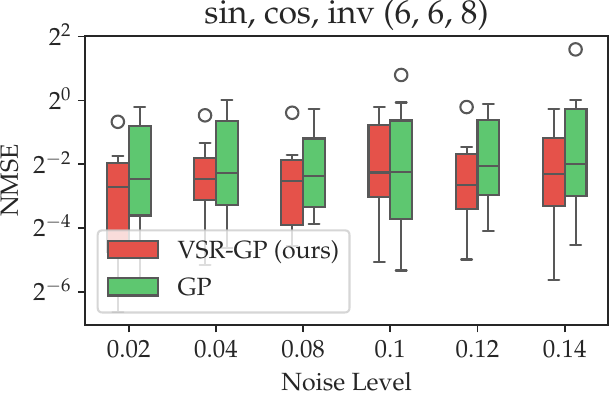} \\
    \caption{Box plots in NMSE values for the expressions found by \method~and GP over benchmark datasets with different noise levels. %\method~consistently finds expressions with smaller NMSEs.
    Our \method is consistently the best regardless of the evaluation metrics and noise levels.
    }
    \label{fig:noise-level-metric}
\end{figure}

\medskip
\textbf{Full Quartile Distribution.} 
We show the full quartiles (25\%, 50\%, and 75\%) over the NMSE metric in noiseless settings (in Fig.~\ref{fig:Quartile-trig-nmse-noiseless-partial}). Here ``$\texttt{inv} (4,4,6)$'' is an abbreviation meaning that the dataset operators $O_p$ is $\{ \text{inv},+,-,\times\}$ and the configuration is $(4,4,6)$. Our \cvgp and \cvmt are consistently the best approaches. This demonstrates that vertical symbolic regression can boost state-of-the-art solvers to an even higher level. %Adding control variable experiment improves NMSE values of best-discovered equations.

\begin{figure}[!t]
    \centering
    \includegraphics[width=0.41\linewidth]{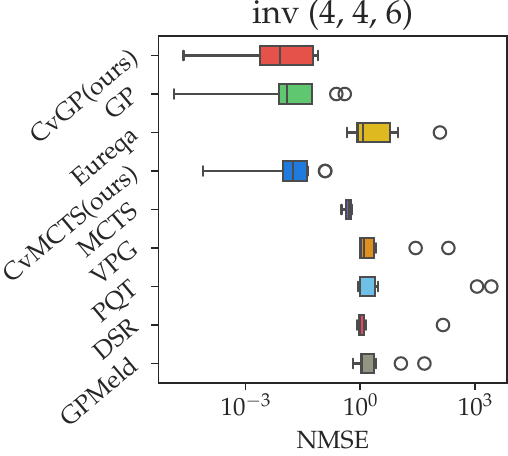}
    \includegraphics[width=0.286\linewidth]{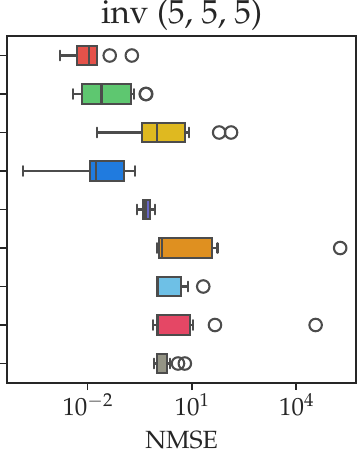}
    \includegraphics[width=0.286\linewidth]{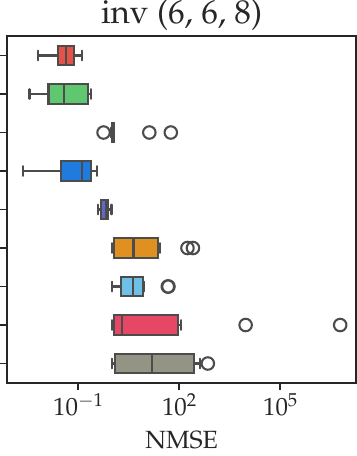} \hfill \\
    \vspace{1em}
    \includegraphics[width=0.405\linewidth]{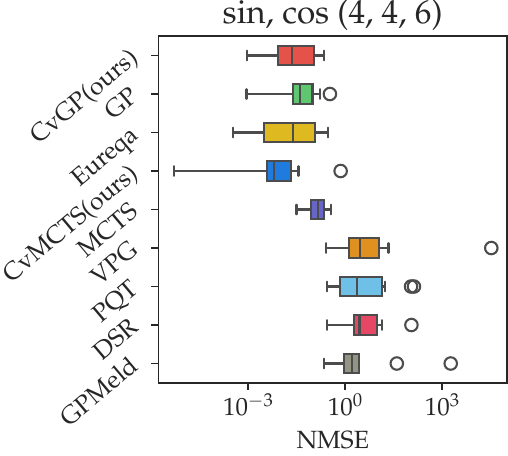}
    \includegraphics[width=0.285\linewidth]{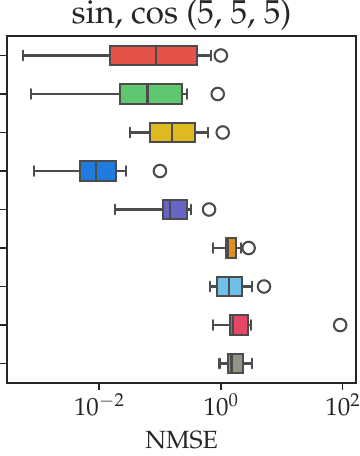}
    \includegraphics[width=0.289\linewidth]{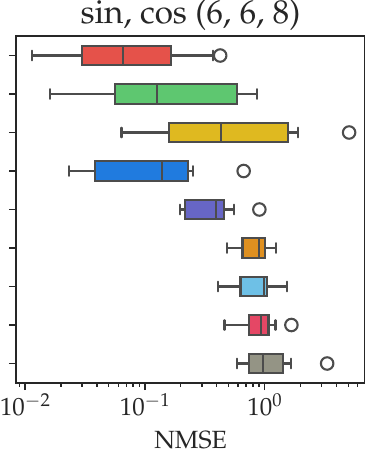} \hfill \\
    \vspace{1em}
    \includegraphics[width=0.41\linewidth]{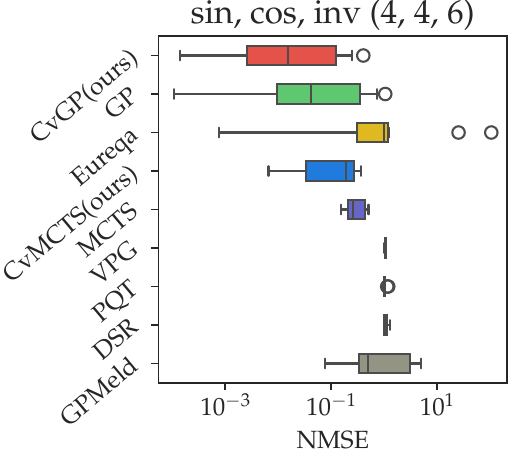}
    \includegraphics[width=0.286\linewidth]{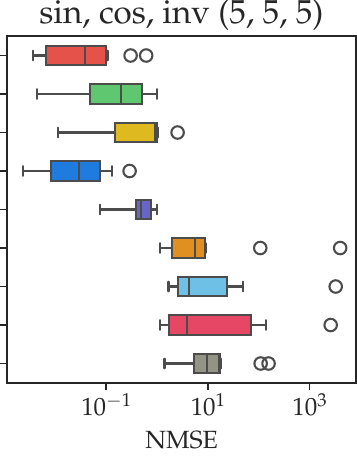}
    \includegraphics[width=0.286\linewidth]{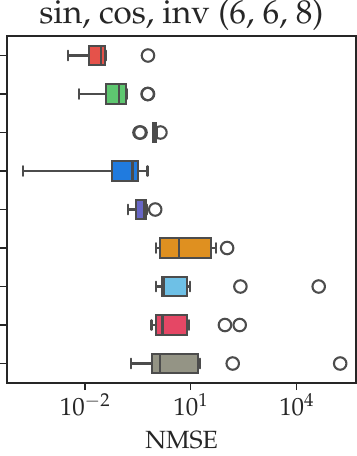}
    \caption{On \textit{noiseless} selected Trigonometric datasets, full quartiles (25\%, 50\% and 75\%) of NMSE values of all the methods. Our \cvgp outperforms GP and our \cvmt outperforms MCTS method. Adding control variable experiment improves NMSE values of best-discovered equations.}
    \label{fig:Quartile-trig-nmse-noiseless-partial}
\end{figure}

\medskip
\textbf{Different Evaluation Metrics.} 
We use box plots in Fig.~\ref{fig:evalucate-metric} to show that the superiority of our \method~approaches generalizes to other quantiles beyond the NMSE metric. 
We evaluate the best expression discovered by each algorithm on more evaluation metrics, including MSE, RMSE, NMSE, NRMSE, and $R^2$-based Accuracy.
Fig.~\ref{fig:evalucate-metric} demonstrates that our approaches \cvgp and \cvmt are the best regardless of the evaluation metrics.

\begin{figure}[!t]
    \centering
    \includegraphics[width=0.41\linewidth]{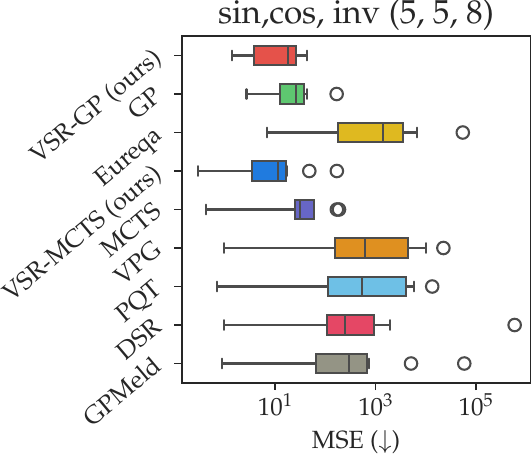}
    \includegraphics[width=0.285\linewidth]{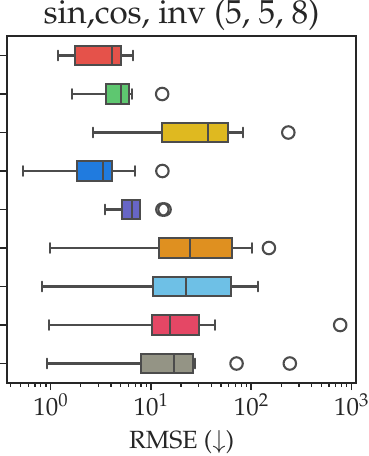}
    \includegraphics[width=0.275\linewidth]{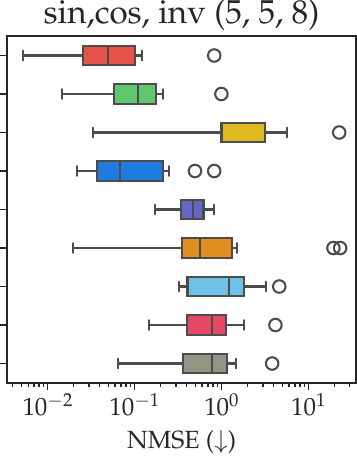}
      \includegraphics[width=0.41\linewidth]{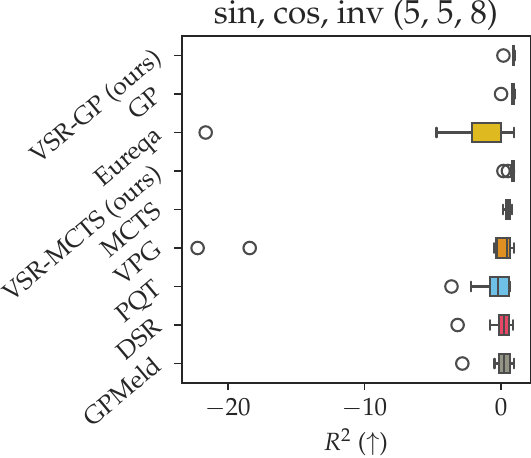}
    \includegraphics[width=0.28\linewidth]{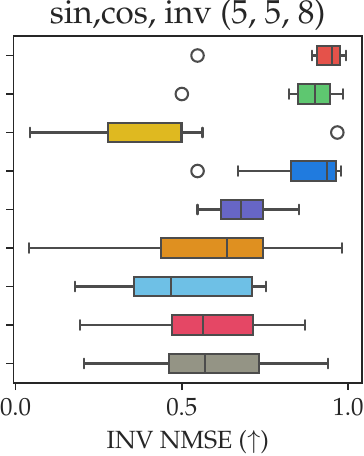}
    \includegraphics[width=0.275\linewidth]{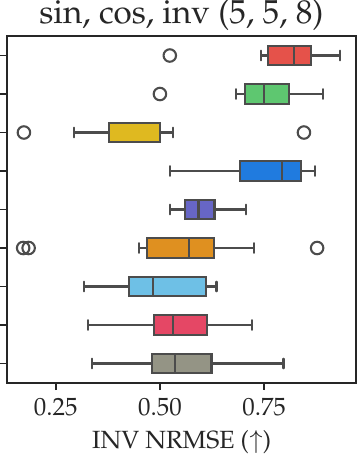}

    \caption{Box plots of different goodness-of-fit metrics for the expressions found by different algorithms on the same \textit{noiseless} dataset. $\uparrow$ indicate a higher score is better and $\downarrow$ indicates a smaller value is better.  Our \cvgp outperforms GP and our \cvmt outperforms MCTS method. Adding control variable experiment improves NMSE values of best-discovered equations.}
    \label{fig:evalucate-metric}
\end{figure}

%% file: tex/9.conclusion.tex
In this research, we propose Vertical  Symbolic Regression (\method) to discover governing equations involving many independent variables from experimental data, which is beyond the capabilities of current state-of-the-art approaches.  
\method~follows a vertical discovery path -- it builds equations involving more and more independent variables using control variable experimentation. Because the first few steps following the vertical discovery route can be exponentially cheaper than discovering the equation in the full hypothesis space (horizontal route), VSP has the potential to supercharge state-of-the-art approaches
in uncovering complex scientific equations with more contributing factors than what current approaches can handle.
Theoretically, we show that \method~can bring an exponential reduction in the search spaces when learning a class of expressions. 
In experiments, we implement two variants of the proposed VSR framework, VSR using Genetic Programming (SR-GP) and VSR using Monte Carlo Tree Search (VSR-MCTS). \method finds the expressions with the smallest median Normalized Mean-Square Errors (NMSE) 
among all 7 competing approaches on noiseless datasets (in Table~\ref{tab:Trigonometric-nmse-noiseless} and Table~\ref{tab:feynman-livermore2-nmse-noiseless}) and 
20 noisy benchmark datasets (in Table~\ref{tab:Trigonometric-nmse-noisy}). In general, \cvgp attains the best empirical results on datasets with a large number of operators while \cvmt is the best on datasets with a median number of operators.
Evaluated on simpler equations, we show that our \method takes less training time and memory consumption, but has a higher rate of recovering the ground-truth expressions compared to baselines following the horizontal paths (in Table~\ref{tab:recovery}). 
We also demonstrate that our \method is consistently better than the baselines under different evaluation metrics (in Fig.~\ref{fig:evalucate-metric}),  different quantiles (25\%, 50\% and 75\%) of the NMSE metric (in Fig.~\ref{fig:Quartile-trig-nmse-noiseless-partial}), and with different amounts of Gaussian noise added to the data (in Fig.~\ref{fig:noise-level-metric}).

%% file: tex/10-expset.tex
\section{Implementation} \label{apx:impelemnt}

\subsection{Implementation of \method}
Our \method incurs several hyper-parameters, they are detailed as follows:
\begin{enumerate}
    \item Threshold for determining the closeness to zero score $\varepsilon=1E-10$. Consistent close-to-zero fitness scores suggest the fitted expression is close to the ground-truth equation in the reduced form.  That is $\sum_{k=1}^K\mathbb{I}(o_k\le \varepsilon)$ should equal to $K$, where $\mathbb{I}(\cdot)$ is an indicator function and  $\varepsilon$ is the threshold for the fitness scores.
    \item Number of multiple trials $K=5$ and the threshold on variances between multiple trials is $\epsilon=0.001$. Given the equation is close to the ground truth, an open constant having similar best-fitted values across trials suggests the open constants are stand-alone. Otherwise, that open constant is a \textit{summary} constant.
The $j$-th open constant is an standalone constant when $\mathbb{I}(\texttt{var}(\mathbf{c}_{l})\le \epsilon)$ is evaluated to $1$, where $\texttt{var}(\mathbf{c}_{l})$ indicates the variance of the solutions for $l$-th open constant. It is computed as $\frac{1}{K}\sum_{k=1}^K(C_{k,l}-\frac{1}{K}\sum_{k=1}^KC_{k,l})^2$. Here $\frac{1}{K}\sum_{k=1}^KC_{k,l}$ computes the averaged fitted value of $l$-th constant and  $ \epsilon$ is the threshold. In the experiment, we set $K=5$ and $\varepsilon=0.01$.
\item The set size of best expressions across all round $|\mathcal{Q}|=50$. It means we only save the top 50 expressions with fitness scores evaluated on non-controlled data.
\end{enumerate}

\begin{figure}[!ht]
    \centering
    \includegraphics[width=1\linewidth]{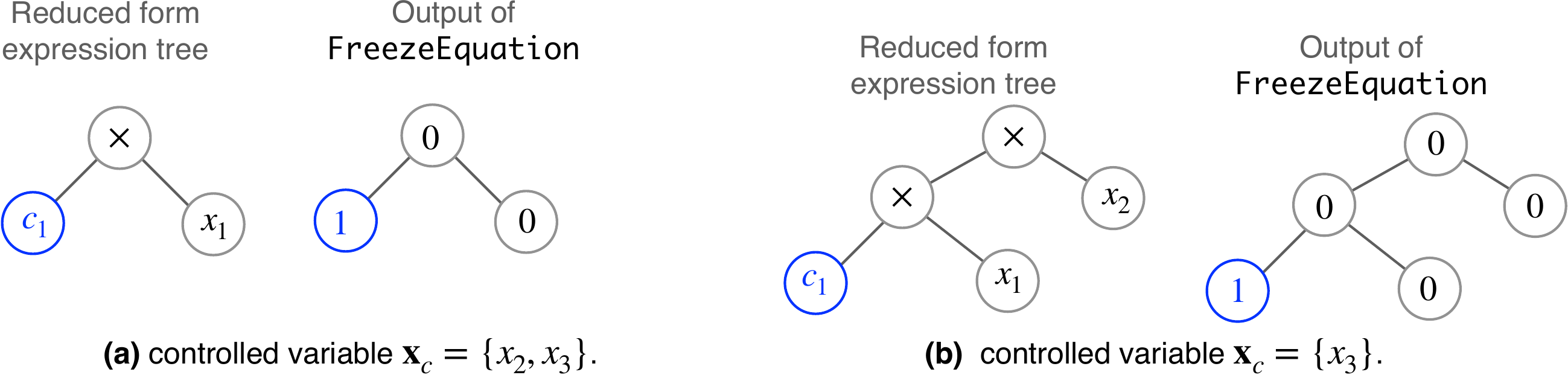}
    \caption{Visualization of the $\mathtt{FreezeEquation}$ function over expression tree in GP. The value ``$0$'' implies the corresponding node is non-editable by mutation or crossover operations in GP. Similarly, ``$1$'' implies the corresponding node is editable.}
    \label{fig:freeze-tree}
\end{figure}

\subsection{Implementation of Data Oracle} \label{apx:data_oracle}

We provide a Python package\footnote{\url{https://bitbucket.org/jiang631/scibench/src/master/src/scibench}.
The data Oracle is accessed by calling from an object of the class ``\texttt{Equation\_evaluator}''. We first will need to initialize an `\texttt{Equation\_evaluator}` object. This can be accomplished by first importing the necessary package:}

\begin{verbatim}
from scibench.symbolic_equation_evaluator import *
\end{verbatim}
{Then in the main program, execute:}
\begin{verbatim}
Eq_eval = Equation_evaluator(input_eq_name)
\end{verbatim}
where the ``\texttt{input\_eq\_name}'' denotes the input file name passed from the command line argument. ``\texttt{Equation\_evaluator}'' provides you with the following three functions:
\begin{itemize}
    \item \texttt{get\_nvars()}.  
This function returns the number of input variables of the symbolic equation.
    \item \texttt{get\_function\_set()}. This function returns the set of operators $O_p$ possibly used in the symbolic equation. 
    \item \texttt{evaluate(X)}. When queried by the input $X$, it returns noisy estimations of $f(X)$. Here, the datatype of $X$ is ``numpy.ndarray''. It is a matrix. The first dimension corresponds to the batch size and the second dimension corresponds to the ``\texttt{number\_of\_variables}''. Each row of $X$ represents one input to the symbolic expression. The output of the evaluation will be a vector, where the $i$-th output is the noisy estimation of $f(X)_i$.
\end{itemize}

\medskip
\textbf{Expression Storage Format.}
The pre-order traversal of the expression tree is used to store the expression. 
In a pre-order traversal of a binary tree, we first visit the current node,  then traverse its left subtrees, and finally traverse its right subtrees.

Here we use an example symbolic expression "$c_1x_1x_2/x_3$" with $c_1=8.31$. The expression to be stored in the file is based on the pre-order traversal. The pre-order traversal of expression is $[\times, \times, c_1, x_1,\div,x_2,x_3]$. Since we know the arity of every mathematical operator, the traversal corresponds to a unique expression tree. We save it in a textual format as follows:
\begin{verbatim}
['mul', 'mul', '8.31', 'x1', 'div', 'x2', 'x3']
\end{verbatim}

Here, 'mul' and 'div' stands for multiplication and division respectively. We further need to: (1) differentiate unary and binary operators for the internal nodes. For example `sine` function is a unary operator, that only takes one input. `mul` is a binary operator that needs two inputs. (2) differentiate constants and variables in the leaves.
Therefore, we define the \text{extended format} based on the pre-order traversal of the expression tree.
\begin{verbatim}
[('mul','binary'), ('mul','binary'), ('8.314', 'const'), ('x1', 'var'), 
 ('div', 'binary'), ('x2', 'var'), ('x3', 'var')]
\end{verbatim}
This implies that ``mul'' is a binary operator, which is denoted as 'binary';
 '8.31' is a constant, which is denoted as ``const'';  '$x_1$', '$x_2$' are input variables, which are denoted as ``var''. Based on this \textit{extended pre-order traversal format}, we can uniquely determine the expression tree and further uniquely determine its symbolic expression.

The final input to the data oracle is a single JSON file, that is:
\begin{verbatim}
{
  'num_vars': 3, 
  'var_domains':[(0, 1), (0, 1), (0, 1)],
  'function_set': ['add', 'sub', 'mul', 'div', 'const'], 
  'equation': [
        ('mul','binary'), ('mul','binary'), ('8.314', 'const'),  
        ('x1', 'var'), ('div', 'binary'), ('x2', 'var'), ('x3', 'var')
    ]
}
\end{verbatim}
where ``num\_vars'' is the number of variables in the symbolic expression.
``function\_set'' represents the set of mathematical operators. The symbolic expression will use a subset of the operators. ``var\_domains'' is the list of input variable domains. ``equation'' is the extended preorder traversal of the expression tree.

We acknowledge that there are many other APIs (such as Sympy) that offer a unified protocol for expression tree representation and expression evaluation. The API provided by us only relies on Numpy and does not include an extra Python library.

\subsection{Baselines Implementation}
We use publicly available implementations for baseline algorithms whenever possible and use our own implementation if not, and make necessary adjustments to ensure compatibility.

\medskip
\textbf{GP.} The implementation is based on the GP baseline in the DSO package. 
However, we re-implemented the code to make it compatible with the data oracle. 

\medskip
\textbf{Eureqa.} This algorithm is currently maintained by the DataRobot webiste\footnote{\url{https://docs.datarobot.com/en/docs/modeling/analyze-models/describe/eureqa.html}}. We use the python API provided to send the training dataset to the DataRobot website and collect the predicted expression after 30 minutes. This website only allows us to execute their program under a limited budget. Due to budgetary constraints, we were only able to test the datasets for the noiseless settings. 
For the Eureqa method, the fitness measure function is negative RMSE. We generated large datasets of size $10^5$ in training each benchmark.

\medskip
\textbf{DSR, PQT, GPMeld.} These algorithms are evaluated based on an implementation in  \footnote{\url{https://github.com/brendenpetersen/deep-symbolic-optimization}}. For every ground-truth expression, we generate a dataset of $10^5$ training samples. Then we execute all these baselines on the dataset with the configurations listed in Table~\ref{tab:baseline-hyper-config}. 
For the four baselines (\textit{i.e.}, PQT, VPG, DSR, GPMeld), the reward function is INV-NRMSE, which is defined as $\frac{1}{1+\text{NRMSE}}$.

\medskip\textbf{MCTS.} We follow the public implementation\footnote{\url{https://github.com/isds-neu/SymbolicPhysicsLearner}}. We change the code to make it compatible with the data oracle, to make a fair comparison.

Note that Wolfram was not considered in this research, because the current ``\texttt{FindFormula}'' function in the Wolfram language only supports searching for single variable expressions. Active Regression over Genetic Programming~\citep{DBLP:conf/gecco/HautPB23} uses active learning to adaptively query data to expedite the learning step, and is not selected for comparison. The reason is its public implementation only handles expressions without open constants and would be not compatible with the datasets in this research.

Note that the tree edit distance between the correct and predicted symbolic expression tree~\citep{matsubara2022rethinking} is not adopted. We use the provided API\footnote{\url{https://github.com/omron-sinicx/srsd-benchmark}} to compute the tree edit distance~\citep{DBLP:journals/siamcomp/ZhangS89} between the predicted expression tree to the correct expression tree. For the correctly predicted expressions, the normalized tree distance is still $0$. The reasons are that 1) different expressions that are numerically evaluated as the same have different expression trees, Sympy will use heuristics to determine the simplest result for unifying different expressions but still cannot handle long and complex expressions. 2) The original definition of tree edit distance treats left and right subtrees differently, however, it is not the case in the expression tree. Thus this metric is omitted.

\begin{table}[!t]
    \centering
    \caption{Major hyper-parameters settings for all the algorithms considered in the experiment.}
    \label{tab:baseline-hyper-config}
    
    \begin{tabular}{r|ccc} 
    \multicolumn{4}{c}{\textbf{(a)} Genetic Programming-based methods.} \\
    \toprule
          & \cvgp & GP&Eureqa  \\\midrule
         Fitness function & NegMSE &NegMSE &NegRMSE  \\ 
         Testing set size & $256$ & $256$ & $ 50, 000$  \\
      \#CPUs for training& 1 & 1 & N/A \\\midrule
        \#genetic generations & 200 & 200  &10,000\\
      
        Mutation Probability & {0.8} & 0.8 \\
        Crossover Probability & {0.8} & 0.8   \\
        \midrule
        \multicolumn{4}{c}{\textbf{(b)} Monte Carlo Tree Search-based methods.} \\
    \toprule
         &  \cvmt & \mcts \\ \midrule
        Fitness function & NegMSE &NegMSE  \\ 
         Testing set size & $256$ & $256$ \\
         \#CPUs for training& 1 & 1 \\
         \midrule 
         \multicolumn{4}{c}{\textbf{(c)} Deep reinforcement learning-based methods.} \\
         \toprule
          &  DSR & PQT & GPMeld \\\midrule
         Reward function  & InvNRMSE & InvNRMSE & InvNRMSE \\ 
        Training set size  & $$ 50, 000$$ & $ 50, 000$ & $ 50, 000$ \\
         Testing set size& $$ 256$$ & $256$ & $ 256$ \\
         Batch size & $1024$ & $1024$ & $1024$ \\
      \#CPUs for training& 8 & 8 & 8 \\
  $\epsilon$-risk-seeking policy  & 0.02 & 0.02 & N/A\\ \midrule
        \#genetic generations & N/A & N/A  & 60 \\
        \#Hall of fame  & N/A& N/A& 25  \\
        Mutation Probability  & 0.5 & N/A &N/A  \\
        Crossover Probability& 0.5 & N/A &N/A   \\
        \bottomrule
    \end{tabular}
\end{table}

\subsection{Hyper-parameter Configurations for Training} \label{sec:hyper-parameter}

We list the major hyper-parameter settings for all the algorithms in Table~\ref{tab:baseline-hyper-config}. In our \method, the threshold to freeze operands in \method is if the MSE to fit a data batch is below 1e-6. The threshold to freeze the value of a constant in \method is if the variance of best-fitted values of the constant across trials drops below 0.001.

Note that if we use the default parameter settings, the GPMeld algorithm takes more than 2 days to train on one dataset. 
Because of such slow performance, we cut the number of genetic programming generations in GPMeld by half to ensure fair comparisons with other approaches.

All the methods are executed on the same server hardware: the CPU is  ``Rome CPUs @ 2.0GHz'',  the maximum memory is 4GB for each program, and the maximum training time is 48 hours.  We specify the number of CPUs used for each program (noted as `` \#CPUs for training'') in Table~\ref{tab:baseline-hyper-config}. The Python version for all the programs is set as Python 3.7 while the DSO requires Python 3.6.15 due to dependency issues on compiling C-related cross-platform packages.

For the memory profile step, we use Python API memray\footnote{\url{https://bloomberg.github.io/memray}}. We only benchmark the peak memory incurred by the step of the algorithm learning stages. The memory of whole program initialization and data oracle initialization are not considered for comparison.

%% file: tex/10.dataset.tex
\begin{table}[!t]
    \centering
    \caption{Example expressions used in our experiments with different dataset configurations and the set of operators.}
    \label{tab:apx-dataset-example}
    \begin{tabular}{c|c}
    
     \multicolumn{2}{c}{\textbf{(a)} Trigonometric datasets containing operators $O_p=\{\texttt{inv},+,-,\times\}$.} \\ 
     \toprule
      Dataset Configs &  Example  expression  \\  \midrule
   (2,1,1) & $0.497-{0.682}/{x_1}-{0.735 x_1}/{x_0}$\\
    (3,2,2) & $-0.603 x_0x_1 + 0.744 x_0 + {0.09 x_1}/{x_2} + 0.562 + {0.582}/{x_2}$ \\
    \midrule
\multicolumn{2}{c}{\textbf{(b)} Trigonometric datasets containing operators $O_p=\{\sin,\cos,+,-,\times\}$.} \\ 
         \midrule
 (2,1,1) & $0.259 x_0 \sin(x_1) + 0.197 x_1 - 0.75$\\
 (3,2,2) & $-0.095 x_0 x_2 + 0.012 x_2 \sin(x_1) - 0.576 x_2 - 0.214 \cos(x_0) - 0.625$\\
 \midrule
 \multicolumn{2}{c}{\textbf{(c)} Trigonometric datasets containing operators $O_p=\{\sin,\cos,\texttt{inv},+,-,\times\}$.} \\ 
        \midrule
(2,1,1) & $0.727\sin(x_0)-0.386+{0.183}/{x_0}$ \\
(3,2,2)& ${0.716 x_0}/{x_2}- 0.063 x_1 + 0.274 x_2 \cos(x_1) - 0.729 +{0.062}/{x_2}$ \\
\bottomrule
    \end{tabular}
\end{table}

\begin{table}[!t]
    \caption{Configuration of the ideal gas law in Physics (in (a)Left) as a symbolic expression (in (a) right and (b)) with experimental data (in (c)). }
    \centering
    \label{table:idea_gas}
        \begin{tabular}{c|p{0.3\textwidth}||p{0.2\textwidth}|l} 
         \multicolumn{4}{c}{\textbf{(a)} The symbolic expression of the ideal gas law.} \\
            \toprule
             \multicolumn{2}{c||}{ \textbf{Physical formula}} &  \multicolumn{2}{c}{ \textbf{Symbolic Expression}}   \\              
             \multicolumn{2}{c||}{$P = \frac{RnT}{V}$} & \multicolumn{2}{c}{$y=\frac{c_1 x_1  x_2}{x_3}$} \\
              \midrule
             \textbf{Symbols} & \textbf{Physical Meaning} & \textbf{Variables}  & \textbf{Variable Domains} \\
            \hline
            $R$  & Ideal gas constant &  Constant $c_1$ & $8.31$  \\
            $n$  & Number of moles & Input variable $x_1$  & $(0.01, 10^4)$  \\
             $T$ & Absolute Temperature &  Input variable $x_2$ & $(10, 10^3)$   \\
            $V$   & Volume & Input variable $x_3$& $(10^{-3}, 10^4)$ \\
              $P$   & Absolute Pressure   &  Output variable $y$ \\
            \bottomrule
        \end{tabular}
\end{table}

\subsection{Dataset Configuration}\label{apx:dataset-config}

% Here we present the exact expression, input range, and function set used in our benchmark and experiment.

\medskip
 \textbf{Trigonometric Dataset.} For example, one ground-truth expression that generates a dataset labeled by the ``\texttt{inv}, $+, -, \times$'' operators with configuration $(2,1,1)$ can be:
\begin{equation}
0.496-\frac{0.682}{x_1}-\frac{0.734 x_1}{x_0}
\end{equation}
This expression contains two variables $x_1,x_2$, a cross term $\frac{x_1}{x_0}$ with a  constant  coefficient $0.734$, a single term $1/x_1$ with a  constant  coefficient $0.682$, and a constant $0.496$. We give more examples of such ground-truth expressions in Table~\ref{tab:apx-dataset-example}.

\medskip
 \textbf{Feynman Dataset.}
The expressions in the Feynman dataset have real-world physical meanings, where the type of variables and the range of input variables are connected to their real-world definitions. We give an example in Table~\ref{table:idea_gas}. The expression for ideal gas law is labeled with the equation ID I.39.22 in Feynman's textbook~\citep{leighton1965feynman}. The expression is formatted as $y=x_0c_0c_1/x_2$ in our uniformed notation, where the correspondence between the variables $\{x_1,x_2,x_3,c_0,y\}$ and $P,b,R,T,V$ is presented at Table~\ref{table:idea_gas}(a). The number of variables and the operator set $O_p$ are given as API to the learning algorithms. The preorder traversal is used for executing the output when the data oracle is queried. Further, the range and type of each variable are provided, which is used for the learning algorithm to sample appropriate input values before the query.

The main difficulty of discovering these equations is mainly determined by the number of variables involved in the expression. Thus we divide all the expressions in the Feynman dataset into 6 groups by the number of variables $K$, where the exact partitions (over the equation ID) are detailed in Table~\ref{table:partition-feynman}.

\begin{table}[!t]
    \caption{Partition of Feynman dataset by the number of input variables $m$.}
    \centering
    \label{table:partition-feynman}
        \begin{tabular}{l|l} 
            \toprule
            Partitions & Equation  ID \\ \midrule
            $K=1$ & I.6.20a, I.29.4, I.34.27, II.8.31, II.27.16, II.27.18, III.12.43 \\
            \midrule
            $n=2$ & I.12.1, I.6.20, I.10.7, I.12.4, I.14.3, I.12.5, I.14.4, I.15.10, I.16.6, I.25.13, I.26.2,\\
            &   I.32.5, I.34.10, I.34.14, I.38.12, I.39.10, I.41.16, I.43.31, I.48.2, II.3.24, II.4.23,  \\
            &  II.8.7, II.10.9,  II.11.28, II.13.17, II.13.23, II.13.34, II.24.17, II.34.29a,  \\
            &  II.38.14, III.4.32, III.4.33, III.7.38, III.8.54, III.15.14, bonus.8, bonus.10\\\midrule
            $n=3$ & I.6.20b, I.12.2, I.15.3t, I.15.3x, bonus.20, I.18.12, I.27.6, I.30.3, I.30.5, I.37.4,\\
            &   I.39.11, I.39.22, I.43.43, I.47.23, II.6.11, II.6.15b, II.11.27, II.15.4, II.15.5,  \\
            &  II.21.32,  II.34.2a, II.34.2, II.34.29b, II.37.1, III.13.18, III.15.12, III.15.27, III.17.37,  \\
            &  III.19.51, bonus.5, bonus.7, bonus.9, bonus.15, bonus.18\\\midrule
            $n=4$ & I.8.14, I.13.4, I.13.12, I.18.4, I.18.16, I.24.6, I.29.16, I.32.17, I.34.8, I.40.1,  \\
            & I.43.16, I.44.4, I.50.26, II.11.20, II.34.11, II.35.18, II.35.21, II.38.3, III.10.19,    \\
            &  III.14.14, III.21.20, bonus.1, bonus.3, bonus.11, bonus.19,\\ \midrule
            $n=5$ & I.12.11, II.2.42, II.6.15a, II.11.3, II.11.17, II.36.38, III.9.52, bonus.4,  \\
            & bonus.12, bonus.13, bonus.14, bonus.16 \\ \midrule
            $6\le n\le 8$ & I.11.19, bonus.2, bonus.17, bonus.6, I.9.18 \\
            \bottomrule
        \end{tabular}
\end{table}

\begin{table*}[!t]
    \small
    \centering
    \caption{Detailed equation in Livermore2 ($n=4$) datasets.}
    \label{tab:livermore2-1}
    \begin{tabular}{cl}
    \toprule
     \multicolumn{2}{c}{\textbf{Livermore2 ($n=4$)}} \\  
        Eq. ID  & Expression \\ \midrule
Vars4-1 & $x_1-x_2x_3-x_2-x_4$ \\
Vars4-2 & $x_1\sqrt{x_2}x_4/x_3$ \\
Vars4-3 & $2x_1+x_4-0.01+x_3/x_2$ \\
Vars4-4 & $x_1-x_4-(-x_1+\sin(x_1))^4/(x_1^8x_2^2x_3^2)$ \\
Vars4-5 & $x_1+\sin(x_2/(x_1x_2^2x_4^2(-3.22x_2x_4^2+13.91x_2x_4+x_3)/2+x_2))^2$ \\
Vars4-6 & $(-x_1-0.54\exp(x_1sqrt(x_4+\cos(x_2))\exp(-2x_1)))/x_3$ \\
Vars4-7 & $x_1+\cos(x_2/\log(x_2^2x_3+x_4))$ \\
Vars4-8 & $x_1(x_1+x_4+\sin((-x_1\exp(\exp(x_3))+x_2)/(-4.47x_1^2x_3+8.31x_3^3+5.27x_3^2)))-x_1$ \\
Vars4-9 & $x_1-x_4+\cos(x_1(x_1+x_2)(x_1^2x_2+x_3)+x_3)$ \\
Vars4-10 & $x_1+(x_1(x_4+(\sqrt{x_2}-\sin(x_3))/x_3))^{1/4}$ \\
Vars4-11 & $2x_1+x_2(x_1+\sin(x_2x_3))+\sin(2/x_4)$ \\
Vars4-12 & $x_1x_2+16.97x_3-x_4$ \\
Vars4-13 & $x_4(-x_3-\sin(x_1^2-x_1+x_2))$ \\
Vars4-14 & $x_1+\cos(x_2^2(-x_2+x_3+3.23)+x_4)$ \\
Vars4-15 & $x_1(x_2+\log(x_3+x_4+\exp(x_2^2)-0.28/x_1))-x_3-\frac{x_4}{2x_1x_3}$ \\
Vars4-16 & $x_3(-x_4+1.81/x_3)+\sqrt{x_2(-x_1^2\exp(x_2)-x_2)}-2.34x_4/x_1$ \\
Vars4-17 & $x_1^2-x_2-x_3^2-x_4$ \\
Vars4-18 & $x_1+\sin(2x_2+x_3-x_4\exp(x_1)+2.96\sqrt{-0.36x_2^3+x_2x_3^2+0.94}$\\
& $+\log((-x_1+x_2)\log(x_2)))$ \\
Vars4-19 & $(x_1^3x_2-2.86x_1+x_4)/x_3$ \\
Vars4-20 & $x_1+x_2+6.21+1/(x_3x_4+x_3+2.08)$ \\
Vars4-21 & $x_1(x_2-x_3+x_4)+x_4$ \\
Vars4-22 & $x_1-x_2x_3+x_2\exp(x_1)-x_4$ \\
Vars4-23 & $-x_1/x_2-2.23x_2x_3+x_2-2.23x_3/\sqrt{x_4}-2.23\sqrt{x_4}+\log(x_1)$ \\
Vars4-24 & $-4.81\log(\sqrt{x_1\sqrt{\log(x_1(x_1x_2+x_1+x_4+log(x_3)))}})$ \\
Vars4-25 & $0.38+(-x_1/x_4+cos(2x_1x_3/(x_4(x_1+x_2x_3)))/x_4)/x_2$ \\
\hline
    \end{tabular}
\end{table*}

\begin{table*}[!t]
    \centering
    \caption{Detailed equations for Livermore2 ($n=5$) dataset.}
    \label{tab:livermore2-2}
    \begin{tabular}{c|l}
    \toprule
 \multicolumn{2}{c}{\textbf{Livermore2 ($n=5$)}} \\  
        Eq. ID & Expression \\ \midrule
Vars5-1	& $	-x_1+x_2-x_3+x_4-x_5-4.75	$ \\
Vars5-2	& $	x_3\left(x_1+x_5+\frac{0.27}{(x_3^2+\frac{(x_2+x_4)}{(x_1x_2+x_2)})}\right)	$ \\
Vars5-3	& $	2x_1x_2x_3+x_5-\sin(x_1\log(x_2)-x_1+x_4)	$ \\
Vars5-4	& $	x_2+x_3x_4+x_5^2+\sin(x_1)	$ \\
Vars5-5	& $	x_5+0.36\sqrt{(\log(x_1x_2+x_3+\log(x_2+x_4)))}	$ \\
Vars5-6	& $	x_1x_4+x_1+x_2+x_5+\sqrt{(0.08x_1/(x_3x_5)+x_3)}	$ \\
Vars5-7	& $	x_1x_5+\sqrt{(x_1x_2\cos(x_1)-x_1/(x_2+x_3+x_4+8.05))}	$ \\
Vars5-8	& $	\sqrt{(x_2)}x_3-x_4-0.07(x_1+(x_1-x_2)\sqrt{(x_2-0.99)})\cos(x_5)	$ \\
Vars5-9	& $	x_1(x_3+(x_1+x_2)/(x_2x_4+x_5))	$ \\
Vars5-10	& $	x_1/(x_4(-0.25x_1x_3x_4+x_2-8.43x_4x_5)\sin(x_3+\log(x_2)))+x_4x_5	$ \\
Vars5-11	& $	-x_4^2+\sqrt{\frac{x_1(x_3+x_5)-x_2+x_5}{x_3}}+0.47\sqrt{x_3\frac{x_1-\sqrt{x_2}+x_2}{x_2}}	$ \\
Vars5-12	& $	x_1\left(x_2-\frac{1}{x_3(x_4+x_5)}\right)	$ \\
Vars5-13	& $	\sqrt{(x_1(x_5(x_2-1.52)-\cos(4.03x_3+x_4)))}	$ \\
Vars5-14	& $	-x_1/(x_2x_5)+\cos(x_1x_3x_4\exp(-x_2))	$ \\
Vars5-15	& $	-x_4+\log(x_1/\log(11.06x_2x_5^2)+x_3)-\cos(x_2+x_5+\sqrt{(x_2x_5)})	$ \\
Vars5-16	& $	x_2+0.33x_5(x_1/(x_1^2+x_2)+x_3x_4^(3/2))	$ \\
Vars5-17	& $	x_1-\sin(x_2)+\sin(x_3)-\cos(-x_2+\sqrt{(x_4)}+x_5)+0.78	$ \\
Vars5-18	& $	x_1x_2-x_4-(\sqrt{(x_3^2/(x_1(x_3+x_4)))}-1.13/x_3)/x_5	$ \\
Vars5-19	& $	4.53x_1x_2+x_1-x_1\cos(\sqrt{(x_2)})/x_2-x_3-x_4-x_5	$ \\
Vars5-20	& $	-\exp(x_1+x_5)+\sin(x_1-4.81)/(0.21(x_5-\log(x_3+x_4)-\exp(x_5))/x_2)	$ \\
Vars5-21	& $	\sqrt{(x_4)}(2x_1+\cos(x_1(x_3x_4\exp(x_1x_2)+x_3-\log(x_3)-3.49))/x_5)	$ \\
Vars5-22	& $	-x_1-x_2+x_3+\sqrt{(x_1-x_2(\sin(x_3)-\log(x_1x_5/(x_2^2+x_4))/x_4))}-0.73	$ \\
Vars5-23	& $	x_1(x_2/(x_3+\sqrt{(x_2(x_4+x_5))}(-x_3+x_4))-x_5)	$ \\
Vars5-24	& $	-x_2x_5+\sqrt{(x_1+x_2(-x_1+x_4\cos(\sqrt{x_3}+x_3)-\frac{x_2+7.84x_3^2x_5}{x_5})+\frac{x_2}{x_3}}	$ \\
Vars5-25	& $	x_1+\log(x_1(-3.57x_1^2x_2+x_1+x_2+x_3\log(-x_1x_4\sin(x_3)/x_5+x_3)))	$ \\
         \bottomrule
    \end{tabular}
\end{table*}

\begin{table*}[!t]
    \centering
    \caption{Detailed equations for Livermore2 ($n=6$) dataset. $n$ stands for the number of maximum variables.}
    \small
    \label{tab:livermore2-3}
    \begin{tabular}{c|l}
\toprule
         \multicolumn{2}{c}{\textbf{Livermore2 ($n=6$)}} \\  
        Eq. ID  & Expression \\ \midrule
Vars6-1	& $x_1-x_6+(x_1+x_4+x_5)\sqrt{(x_1^2+x_2-x_3)}$ \\
Vars6-2	& $	x_1(2x_2+x_2/x_3+x_4+\log(x_1x_5x_6))	$ \\
Vars6-3	& $	\sqrt{(x_2+x_5-x_6+x_3^4x_4^4/(x_1x_2^4))}	$ \\
Vars6-4	& $	x_1(x_2(x_1^2+x_1)-x_2+x_3^2-x_3-x_5-x_6-\sin(x_4)-\cos(x_4))^2	$ \\
Vars6-5	& $	x_2\sqrt{(x_1x_2)(x_1x_3-x_3-x_4)+x_5+x_6}	$ \\
Vars6-6	& $	(x_1/(x_2x_3+\log(\cos(x_1))^2)-x_2x_4+\sin((x_2x_4+x_5)/x_6)+\cos(x_3))\log(x_1)	$ \\
Vars6-7	& $	x_1\sqrt{(x_1-x_6^2+\sin((x_1\exp(-x_2)-x_4(x_2+x_3^2))/(x_2+x_5)))}	$ \\
Vars6-8	& $	x_1+x_2^2+0.34x_3x_5-x_4+x_6	$ \\
Vars6-9	& $	x_4(x_1+\exp(13.28x_3^2x_6-x_5^2\log(x_2^4))/(x_1x_3-x_2^2+x_2-x_6-\log(x_3)))	$ \\
Vars6-10	& $	x_1+61.36x_2^6+x_2/(x_1x_3(x_4-\cos(x_4(2x_1x_2x_6/x_5+x_5))))	$ \\
Vars6-11	& $	(x_1+x_1/(x_2+x_4(8.13x_1^2x_6+x_1x_2x_3+2x_2+x_5+x_6)))^2	$ \\
Vars6-12	& $	(\sqrt{2}\sqrt{(x_1)}-x_2-x_3/\sqrt{(x_4(8.29x_1x_3^2+x_1x_5)+x_4+x_6)})/x_6	$ \\
Vars6-13	& $x_1+x_5+0.21\sqrt{(x_1/(x_2^2x_3^2\sqrt{(x_6)}(\sqrt{(x_3)}+x_3+2x_6+(x_2+x_4+x_5)/x_5)))}	$ \\
Vars6-14	& $	-2.07x_6+\log(x_2-x_6-\sqrt{(x_3(x_5+\log(-x_1+x_5+1))/x_4)})	$ \\
Vars6-15	& $	x_1(x_1+\cos(x_2^2x_3x_4(x_5-0.43x_6^2)))/x_4	$ \\
Vars6-16	& $	-\sqrt{(x_1)}-x_1+x_2-x_4-x_5-\sqrt{(x_6/x_3)}-3.26	$ \\
Vars6-17	& $	x_1/(x_2x_4(-x_5+\log(x_6^2\cos(2x_2+x_3^2-x_4)^2))(129.28x_1^2x_2^2+x_3))	$ \\
Vars6-18	& $	\sqrt{x_5}(2x_1+\cos(x_1(x_3x_4\exp(x_1x_2)+x_3-\log(x_3)-3.49))/x_6)	$ \\
Vars6-19	& $	x_1+x_2+x_3+0.84\sqrt{-x_3x_6+x_4-x_5+\sqrt{(x_2+\log(x_3+\exp(x_2)))/(x_2-x_4)}}$ \\
Vars6-20	& $	(x_1-0.97x_1/(x_5-x_6(x_1^(3/2)x_4+x_6))-x_2+x_3+\sin(x_1^2)/x_1)^2	$ \\
Vars6-21	& $	x_1+x_3+(x_1+\sin(-3.47x_2\log(x_6)/x_5+x_4+25.56\exp(x_5^2)/x_2)^2)\sin(x_2)	$ \\
Vars6-22	& $	x_1+(x_4+\sin(-0.22(x_3-x_4+1.0))\cos(x_6))\cos(x_2+2.27x_5)	$ \\
Vars6-23	& $	x_1+x_4+\log(x_1^2+x_1(-x_6+1.88\sqrt{(0.71x_1+x_2+0.28(x_3-x_4/x_5))}))	$ \\
Vars6-24	& $	-0.59(1.42(0.24x_2+\sqrt{x_3}/(x_6\sqrt{(-x_4+x_5)}))^(1/4)+\sin(x_1))/x_6	$ \\
Vars6-25	& $	x_1-x_2^2-x_3+x_5\cos(x_3)+x_5+x_6-2.19\sqrt{(-x_3-0.44/x_4)}	$ \\
\hline
    \end{tabular}
\end{table*}

\subsection{Livermore2 Dataset}
In Tables~\ref{tab:livermore2-1},~\ref{tab:livermore2-2},~\ref{tab:livermore2-3}, we detail the exact equations of Livermore2~\citep{DBLP:conf/iclr/PetersenLMSKK21}. The set of operators used for each expression is different, please see the data folder in the publicly available code base.